\documentclass{article}
\usepackage[english]{babel}
\usepackage[letterpaper,top=2cm,bottom=2cm,left=3cm,right=3cm,marginparwidth=1.75cm]{geometry}

\usepackage{authblk}
\usepackage{algorithm,algorithmic,amsthm}

\usepackage{threeparttable}
\usepackage{hyperref}
\usepackage{natbib}

\usepackage[utf8]{inputenc} 
\usepackage[T1]{fontenc}    
\usepackage{hyperref}       
\usepackage{url}            
\usepackage{booktabs}       
\usepackage{amsfonts}       
\usepackage{nicefrac}       
\usepackage{microtype}      
\usepackage{algorithm,algorithmic}
\usepackage{xcolor}         
\usepackage{multirow}

\usepackage{amsmath,amscd,amsbsy,amssymb,latexsym,url,bm}
\allowdisplaybreaks
\usepackage{cleveref}
\usepackage{verbatim}
\usepackage{color}
\usepackage{dsfont}

\newtheorem{theorem}{Theorem}

\newtheorem{lemma}{Lemma}

\newtheorem{definition}{Definition}

\newcommand{\bOne}[1]{\mathds{1} \! \left\{#1\right\}}

\mathchardef\mhyphen="2D

\usepackage{hyperref}
\hypersetup{
  colorlinks   = true, 
  urlcolor     = blue, 
  linkcolor    = blue, 
  citecolor   = blue 
}

\usepackage{mathtools}
\usepackage{dsfont}
\usepackage{algorithm}
\usepackage{algorithmic}
\usepackage{bm}
\usepackage{booktabs}
\usepackage{array}

\usepackage{graphicx}
\graphicspath{{Figures/}}
\usepackage{caption}
\usepackage[export]{adjustbox}
\usepackage{float}
\usepackage[all]{hypcap} 

\newcommand{\E}{\mathbb{E}}

\newcommand{\1}{\mathbb{I}}

\newcommand{\hyb}{\mathrm{hyb}}

\newcommand{\Arms}{[K]}
\newcommand{\Tmax}{T}

\newcommand{\muon}[1]{\mu^{\mathrm{(on)}}_{#1}}
\newcommand{\muoff}[1]{\mu^{\mathrm{(off)}}_{#1}}

\newcommand{\muhybhat}[2]{\hat{\mu}^{\mathrm{(hyb)}}_{#1}(#2)}
\newcommand{\muonhat}[2]{\hat{\mu}^{\mathrm{(on)}}_{#1}(#2)}
\newcommand{\muoffhat}[1]{\hat{\mu}^{\mathrm{(off)}}_{#1}}

\begin{document}

\title{Sample-Mean Anchored Thompson Sampling for Offline-to-Online Learning with Distribution Shift}

\author{
\begin{minipage}{0.45\textwidth}
\centering
Bochao Li\par
Southern University of Science and Technology\par
\vspace{-0.15em}
{\small\texttt{12531179@mail.sustech.edu.cn}}
\end{minipage}
\hfill
\begin{minipage}{0.45\textwidth}
\centering
Yao Fu\par
Southern University of Science and Technology\par
\vspace{-0.15em}
{\small\texttt{12211653@mail.sustech.edu.cn}}
\end{minipage}

\vspace{1.0em}

\begin{minipage}{0.45\textwidth}
\centering
Wei Chen\par
Microsoft Research\par
\vspace{-0.15em}
{\small\texttt{weic@microsoft.com}}
\end{minipage}
\hfill
\begin{minipage}{0.45\textwidth}
\centering
Fang Kong\thanks{Corresponding author.}\par
Southern University of Science and Technology\par
\vspace{-0.15em}
{\small\texttt{kongf@sustech.edu.cn}}
\end{minipage}
}
\date{}

 \maketitle

\begin{abstract} 
Offline-to-online learning aims to improve online decision-making by leveraging offline logged data. A central challenge in this setting is the distribution shift between offline and online environments. 
While some existing works attempt to leverage shifted offline data, they largely rely on UCB-type algorithms. Thompson sampling (TS) represents another canonical class of bandit algorithms, well known for its strong empirical performance and naturally suited to offline-to-online learning through its Bayesian formulation. However, unlike UCB indices, posterior samples in TS are not guaranteed to be optimistic with respect to the true arm means. This makes indices constructed from purely online and hybrid data difficult to compare and complicates their use.
To address this issue, we propose sample-mean anchored TS (Anchor-TS), which introduces a novel median-based anchoring rule that defines the arm index as the median of an online posterior sample, a hybrid posterior sample, and the online sample mean. The median anchoring systematically corrects bias induced by distribution shift by mitigating over-estimation for suboptimal arms and under-estimation for optimal arms, while exploiting offline information to obtain more accurate estimates when the shift is small.
We establish theoretical guarantees showing that the proposed algorithm safely leverages offline data to accelerate online learning, and quantifying how the degree of distribution shift and the size of offline data affect the resulting regret reduction. 
Extensive experiments demonstrate consistent improvements of our algorithm over baselines.
\end{abstract}

\section{Introduction}


Stochastic multi-armed bandits (MAB) provide a fundamental framework for sequential decision-making under uncertainty \citep{lai1985asymptotically, lattimore2020bandit}, where a learning agent repeatedly selects actions to learn the unknown environment and maximize cumulative reward. The upper confidence bound (UCB) \citep{Auer2002} and Thompson sampling (TS) \citep{agrawal2013ts} are two canonical algorithm types to solving the problem: the former relies on optimism-driven confidence intervals, and the latter selects actions by sampling from posterior distributions. 
Despite their importance, classical bandit algorithms typically start from scratch. However, learning purely from online interactions can be costly or risky. Many applications provide access to offline logged data collected by historical policies, which motivates the offline-to-online learning setting that leverages such data to accelerate subsequent online learning \citep{nair2020awac, lee2022offline, shivaswamy2012multi, wagenmaker2023leveraging}.


Recently, there has been increasing interest in offline-to-online algorithms under distribution shift \citep{CheungLyu2024, yin2025multi, he2024learning, qu2025hybrid}, where the offline reward distribution differs from the online environment due to system updates, non-stationarity, or sim-to-real gaps. A recurring insight in this line of work is that achieving performance no worse than purely online learning requires some form of prior knowledge about the distribution shift. Existing approaches mainly adopt UCB-type algorithms. A representative strategy is to construct two UCBs: a hybrid UCB that incorporates offline data together with a bias correction term, and a purely online UCB based solely on online observations \citep{CheungLyu2024}. Actions are selected according to the minimum of the two, which guarantees regret no worse than that of pure online UCB.

Thompson sampling (TS) represents another canonical class of bandit algorithms and has been extensively studied across a wide range of online learning settings \citep{agrawal2012analysis,daniel2018tutorial}. It is well known for its strong empirical performance, often outperforming UCB-based methods in practice, and its Bayesian formulation makes it particularly appealing for offline-to-online learning as unbiased offline data can be naturally incorporated into the prior distribution \citep{oetomo2023cutting, agnihotri2024online}.

Despite these advantages, the posterior-sampling nature of TS makes it fundamentally more challenging under distribution shift. In UCB-based methods, the index is optimistic, and regret can be attributed to inaccurate estimates of suboptimal arms. By contrast, TS relies on posterior samples that may fall on either side of the true arm means, and its regret depends jointly on the accuracy of both the optimal and suboptimal arms. 
This principal difference has important implications for incorporating offline data. For UCB-based methods, conservatively taking the minimum of a purely online index and a hybrid index provides a safe guard against distribution shift. For TS, however, no analogous comparison rule exists: taking the minimum risks underestimating the optimal arm, while taking the maximum may overestimate suboptimal arms. This lack of a principled mechanism for comparing posterior samples makes it nontrivial to exploit biased offline data for TS.
To address this challenge, we propose sample-mean anchored Thompson sampling (Anchor-TS), a modified TS algorithm based on a median-of-indices aggregation scheme. At each round, Anchor-TS assigns to each arm an index defined as the median of three statistics: (i) an online posterior sample, (ii) a hybrid posterior sample that integrates offline and online data, and (iii) the online sample mean, which serves as an unbiased anchor. This median-based aggregation induces an inherent robustness–efficiency trade-off. When the offline data are reliable, the hybrid posterior sample typically exhibits reduced variance and remains close to the anchor, thereby accelerating learning. Conversely, under severe bias in the offline component, the hybrid sample tends to behave as an outlier and is filtered out by the median, so that the resulting index is primarily driven by online observations. 
Furthermore, to prevent the underestimation of the optimal arms' online sample mean from being amplified under median aggregation when its offline reward is smaller, we additionally introduce a right-hand shift to the hybrid posterior distribution.

Our analysis develops a new regret decomposition tailored to median-based TS. In particular, we analyze the behavior of posterior samples for both suboptimal and optimal arms relative to the online sample mean, which serves as an anchor, and characterize their joint contributions to regret accordingly. This enables us to bound the leading regret terms in a way that captures the more favorable behavior between online TS and hybrid TS. Specifically, we obtain a regret upper bound 
\begin{multline*}
O\Bigg( \sum_{i\in[K]\setminus\{1\}} \Delta_i \Big( 
\big(\log T/\Delta_i^2 - N_i \max\{0, 1-3\omega_i/\Delta_i\}\big)_+
+ \big(\log T/\Delta_i^2 - N_1\big)_+ \Big) \Bigg).
\end{multline*}
%
%
where $T$ is online horizon, $K$ is the arm set size, arm $1$ is the optimal arm, $\Delta_i$ is arm $i$'s sub-optimality gap compared with $1$, $\omega_i$ is an effective discrepancy related to $i$'s offline and online expected reward, $N_i$ is the offline sample size of $i$, $(\cdot)_+$ represents $\max\{0,\cdot\}$. 

The above regret guarantee formalizes the robustness–efficiency trade-off under distribution shift: it is no worse than the pure online TS \citep{agrawal2013ts}, and strictly improves when the distribution shift $\{\omega_i\}_{i\in [K]}$ is mild. Compared to existing UCB-based algorithms \citep{CheungLyu2024}, the regret reduction attributable to suboptimal arms estimations is of similar order. Crucially, our analysis further reveals an additional source of improvement that is unique to TS: offline data also reduces regret by accelerating concentration on the optimal arm, as reflected in the dependence on $N_1$. Such an advantage is particularly relevant in practice, where offline data are often collected by prior learning policies or expert behavior and therefore tend to contain a large number of samples from the optimal arm. In such cases, Anchor-TS can effectively exploit this abundance of optimal-arm data, whereas UCB-based methods are unable to benefit in the same way. 
Extensive empirical results further validate our theory. Across all considered settings, Anchor-TS consistently and substantially outperforms UCB-based algorithms. The performance gap becomes even more pronounced when more offline data are on the optimal arm. Importantly, even in the unbiased and pure online settings, the sample-mean–based mechanism helps reduce the excessive exploration caused by the high variance of a single posterior sample in vanilla TS. As a result, it achieves better empirical performance than the corresponding TS baselines in these settings.

\section{Related Work}\label{sec:related work}

\paragraph{Thompson sampling (TS). }
TS is a canonical Bayesian framework for stochastic MAB. 
It has been widely adopted since the seminal work of \cite{thompson1933likelihood}, but the establishment of convergent regrets lagged for decades \citep{agrawal2012analysis, kaufmann2012thompson, agrawal2013ts, jin2021mots, jin2023}. In many problem settings, although UCB-based algorithms have been applied, researchers continue to devote significant effort to TS-type algorithms \citep{agrawal2013thompson, komiyama2015optimal, wang2018thompson, verstraeten2020multi, zhang2020neural}, driven by their superior empirical performance and ease of implementation \citep{granmo2010solving, scott2010modern, chapelle2011empirical}. 
More recently, TS-type algorithms are also proposed for offline-to-online learning problems with offline data naturally encoded in the prior \citep{oetomo2023cutting,agnihotri2024online}. These approaches typically do not account for distribution shift between offline and online environments.

\paragraph{Offline-to-online Learning. }
Offline-to-online learning aims to accelerate online decision-making by leveraging pre-collected logged data. This paradigm has been extensively explored empirically in both bandit and reinforcement learning (RL) settings \citep{lee2022offline, nair2020awac, ball2023efficient, yu2023actor, nakamoto2023cal, xia2024toward}, and has also motivated a growing body of theoretical work establishing performance guarantees. 

Early theoretical studies primarily focused on the unbiased setting, where offline and online reward distributions coincide.
In the bandit setting, \citet{shivaswamy2012multi} show that historical data can yield constant regret, a result later extended to contextual bandits through informative priors \citep{oetomo2023cutting} and meta-algorithmic approaches \citep{banerjee2022artificial}. More recently, \citet{sentenac2025balancing} analyze the trade-off between optimism in online learning and pessimism in offline learning in the offline-to-online context.
For RL, theoretical frameworks have expanded from tabular Markov decision processes \citep{xie2021policy} to linear models \citep{wagenmaker2023leveraging, tan2024natural, huang2025augmenting} and further to general function approximation \citep{songhybrid}, with a common focus on relaxing offline data coverage assumptions.  

In practice, offline data often originate from heterogeneous or outdated sources, making it essential to achieve both efficiency gains from informative offline data and robustness to distribution shift. Fundamental lower bounds establish that without a priori knowledge of the distribution shift, no policy can uniformly outperform purely online algorithms \citep{zhang2019warm, CheungLyu2024, zhang2025contextual, qu2025hybrid}. Consequently, existing works typically incorporate some shift information into the decision process. Such information is encoded through lower bounds \citep{qu2025hybrid} or upper bounds on the shift \citep{CheungLyu2024,ahn2025online}. Our work follows the last line of works. A representative strategy in this line is to construct a hybrid UCB index by combining offline data with a known upper bound on the bias, and to select actions according to the minimum of this hybrid index and a purely online UCB \citep{CheungLyu2024}. This strategy is applied to a variety of settings, including best arm identification (BAI) \citep{yang2025best}, heterogeneous feedback structures \citep{he2024learning}, auxiliary rewards \citep{yin2025multi},  combinatorial MAB \citep{zhou2025hybrid}, and linear bandits \citep{zhang2025contextual}. To the best of our knowledge, TS-type algorithms for offline-to-online learning under distribution shift is still open. 

\paragraph{TS with misspecified priors. }Our work is also related to studies on misspecified priors in TS. 
\citet{liu2016prior,simchowitz2021bayesian} show that when TS is initialized with strongly biased priors,
	where the prior assigns very little probability to the true model, it can suffer linear regret and perform much worse than methods with uninformative priors. 
These results demonstrate the negative impact of using an incorrect prior. 
Compared with this line of work, our setting explicitly models the prior as being induced by offline data. 
Our setting requires not only robustness to biased priors, but also the ability to achieve regret improvement when offline data are informative.


\section{Preliminaries}
\label{sec: preliminaries}

We begin with the classical stochastic multi-armed bandit problem. An agent interacts with an environment over a time horizon $\Tmax$ by repeatedly selecting arms from a finite set $\Arms = \{1,\dots,K\}$. Each arm $i\in\Arms$ is associated with an unknown reward distribution $P_i^{\mathrm{(on)}}$ supported on $[0,1]$\footnote{Our analysis extends to sub-Gaussian reward distributions by replacing Hoeffding-type concentration bounds with their sub-Gaussian counterparts. The regret order remains unchanged.} and unknown mean $\muon{i}$. At each round $t=1,2,\dots,\Tmax$, the agent selects an arm $A(t)\in\Arms$ and observes a reward $R_{A(t)}(t)$ drawn independently from the corresponding online distribution $P_{A(t)}^{\mathrm{(on)}}$. Without loss of generality, let arm $1 \in \arg \max_{j\in\Arms} \muon{j}$ denote the unique\footnote{The uniqueness assumption is only for the convenience of
		the analysis. The setting with multiple optimal arms can only decrease the regret as shown in \cite{agrawal2012analysis}. } optimal arm. 
For each suboptimal arm $i\neq 1$, define the sub-optimality gap $\Delta_i = \muon{1} - \muon{i} $.

In this work, we study a stochastic bandit setting augmented with
\emph{offline data}.
In addition to the online interaction described above, each arm $i\in\Arms$
is associated with an offline dataset
$S_i = \{X_{i,1},\dots,X_{i,N_i}\}$ of size $N_i$.
The samples in $S_i$ are assumed to be i.i.d. draws from an
offline distribution $P_i^{\mathrm{(off)}}$ supported on $[0,1]$, with unknown mean $\muoff{i} = \E_{X_i\sim P_i^{\mathrm{(off)}}}\left[X_i \right]$.


Importantly, for any arm $i$, the offline and online reward distributions are not assumed to coincide. To explicitly model the potential distribution shift between them, we assume that for each arm $i$, 
there exists a known upper bound $V_i \ge 0$ such that $|\muoff{i} - \mu_i^{\mathrm{(on)}}| \le V_i$. This assumption provides a priori control over the magnitude of the distribution shift and is information-theoretically necessary. Existing works establish fundamental lower bounds that no policy can uniformly outperform purely online algorithms without any prior knowledge \citep{zhang2019warm, CheungLyu2024, zhang2025contextual, qu2025hybrid} and such upper bound is widely adopted in previous works \citep{CheungLyu2024,zhang2025contextual,he2024learning,ahn2025online,yin2025multi}.

The performance of a policy $\pi$ is measured by its cumulative regret,
which quantifies the expected loss relative to always pulling the optimal arm:
\begin{equation}
  \label{eq:reg-def}
  \mathrm{Reg}_\pi(\Tmax)
  = \Tmax \mu_1^{\mathrm{(on)}}
  - \E\!\left[\sum_{t=1}^{\Tmax} R_{A(t)}(t)\right] = \sum_{i\neq 1}\Delta_i \E\left[ \sum_{t=1}^T \bOne{A(t)=i} \right] \,,
\end{equation}
where the expectation is taken over all sources of randomness, including the internal randomness of $\pi$, the offline data draw from $P^{\mathrm{(off)}}$, and the online rewards generated under $P^{\mathrm{(on)}}$. 

\noindent\textbf{Useful Notations. } For each arm $i\in[K]$, let $T_i(t)$ denote the number of times arm $i$ is pulled
during rounds $1,\dots,t-1$, and define $n_i(t) = T_i(t) + N_i$ as the total number of samples available for arm $i$, combining online and
offline data.
Let $\muoffhat{i}$ denote the sample mean of the offline dataset for arm $i$,
and let $\muonhat{i}{t}$ be the sample mean computed from the online rewards
observed up to round $t-1$.
We further define $\muhybhat{i}{t}$ as the hybrid sample mean that aggregates
both offline samples and online observations for arm $i$. 
\section{Algorithm}\label{sec: alg}

In the offline-to-online setting, it is natural to maintain two types of estimators: a purely online estimator constructed solely from online observations, and a hybrid estimator that aggregates offline data with online observations. Due to the distribution shift, the hybrid estimators may be biased and can mislead the learning process if used indiscriminately. The key algorithmic challenge is therefore how to balance the use of the hybrid estimator against the purely online estimator. In this section, we present our sample-mean anchored TS (Anchor-TS, Algorithm \ref{alg:hyb-ts}), an efficient and robust mechanism for combining online and hybrid estimators under the TS framework.

\begin{algorithm}[tbh!]
  \caption{Sample-Mean Anchored Thompson sampling (Anchor-TS)}
  \label{alg:hyb-ts}
  \begin{algorithmic}[1]
    \STATE \textbf{Input:} Arm set $[K]$; 
           offline sample mean $\hat\mu_i^{\mathrm{(off)}}$ and sample size $N_i$, 
           bias bound $V_i$,  $\forall i\in[K]$
    \STATE Initialization: $T_i(1) \gets 0$, $\hat\mu_i^{(\mathrm{on})}(1) \gets 0$, $\hat\sigma_{i,\mathrm{on}}^2(t) \gets 1$,  $\hat\mu_i^{\mathrm{(hyb)}}(1) \gets \hat\mu_i^{\mathrm{(off)}}$, $\hat\sigma_{i,\mathrm{hyb}}^2(1)
        \gets
        {1}/{(N_i+1)}$, $Z_{i,1}\gets V_i$, $\forall i\in[K]$
    \FOR{$t=1,\dots$}
      \FOR{each arm $i\in[K]$}
        \STATE Sample $\theta_i^{\mathrm{(on)}}(t)$ from $\mathcal{N}\left(\hat\mu_i^{(\mathrm{on})}(t),\; \hat\sigma_{i,\mathrm{on}}^2(t) \right)$  \label{alg:line:theta_on}
        \STATE Sample $\theta_i^{\mathrm{(hyb)}}(t)$ from  $\mathcal{N}\left(\hat\mu_i^{(\mathrm{hyb})}(t)+Z_{i,t},\; \hat\sigma_{i,\mathrm{hyb}}^2(t) \right)$      \label{alg:line:theta_hybrid}     
        \STATE $\hat\theta_i(t)
                \gets
                \mathrm{median}\bigl\{
                  \hat\mu_i^{\mathrm{(on)}}(t),\
                  \theta_i^{\mathrm{(on)}}(t),\
                  \theta_i^{\mathrm{(hyb)}}(t)
                \bigr\}$  \label{alg:line:median}
      \ENDFOR
      \STATE Select arm $A(t) \gets \arg\max_{i\in[K]} \hat\theta_i(t)$ and observe reward $R_{A(t)}(t)$\label{alg:line:select}
      \STATE \textbf{// Update online posterior of $A(t)$}
      \STATE $T_{A(t)}(t+1) \gets T_{A(t)}(t) + 1$  
      \STATE $
              \hat\mu_{A(t)}^{\mathrm{(on)}}(t+1)
              \gets
              \frac{{}T_{A(t)}(t)\cdot \hat\mu_{A(t)}^{\mathrm{(on)}}(t) + R_{A(t)}(t)}
                   {T_{A(t)}(t+1)+1}$,
      \quad $
        \hat\sigma_{A(t),\mathrm{on}}^2(t+1)
              \gets \frac{1}{T_{A(t)}(t+1)+1}$ \label{alg:eq:update:on}
      \STATE \textbf{// Update hybrid posterior of $A(t)$}
      \STATE $
              \hat\mu_{A(t)}^{\mathrm{(hyb)}}(t+1)
              \gets
              \frac{
                T_{A(t)}(t+1)\cdot\hat\mu_{A(t)}^{\mathrm{(on)}}(t+1)
                + N_{A(t)}\cdot\hat\mu_{A(t)}^{\mathrm{(off)}}
              }{T_{A(t)}(t+1)+N_{A(t)}+1}$ \label{alg:eq:update:hybrid:mean}
       \STATE $
    \hat\sigma_{A(t),\mathrm{hyb}}^2(t+1)
              \gets
              \frac{1}{N_{A(t)}(t+1)+T_{A(t)}+1}
              $,\quad $Z_{A(t),t+1} \gets \frac{N_{A(t)}V_{A(t)}}{T_{A(t)}(t+1)+N_{A(t)}}$ \label{alg:eq:update:hybrid:variance}
        \FOR{$i\neq A(t)$}
            \STATE Update
             $\hat\mu_i^{\mathrm{(on)}}(t+1)\gets\hat\mu_i^{\mathrm{(on)}}(t)$,
             $\hat\mu_i^{\mathrm{(hyb)}}(t+1)\gets\hat\mu_i^{\mathrm{(hyb)}}(t)$,
             $\hat\sigma_{i,\mathrm{on}}^2(t+1)\gets\hat\sigma_{i,\mathrm{on}}^2(t)$,
             $\hat\sigma_{i,\mathrm{hyb}}^2(t+1)\gets\hat\sigma_{i,\mathrm{hyb}}^2(t)$,\quad $Z_{i,t+1}\gets Z_{i,t}$
        \ENDFOR
    \ENDFOR
  \end{algorithmic}
\end{algorithm}

For each arm $i\in[K]$, the algorithm maintains an online posterior that is updated exclusively based on rewards observed during online interaction, which is the same as the traditional TS algorithm for the purely online setting \citep{agrawal2013ts}.
At time $t$, this posterior is characterized by the online sample mean $\hat{\mu}_i^{\mathrm{(on)}}(t)$ and an associated variance ${\sigma}_{i,\mathrm{on}}^{2}(t)$ (Line \ref{alg:eq:update:on}).
An online posterior sample ${\theta}_i^{\mathrm{(on)}}(t)$ is drawn from the corresponding Gaussian distribution (Line \ref{alg:line:theta_on}).
Because it relies solely on online data, this posterior remains unbiased with respect to the true arm mean, although it may exhibit high variance in the early stages of learning.

In parallel, Anchor-TS constructs a hybrid posterior that integrates offline empirical information with online observations. Specifically, the hybrid posterior mean $\hat{\mu}_i^{(\mathrm{hyb})}(t)$ is obtained by combining the offline sample mean $\hat{\mu}_i^{(\mathrm{off})}$ and the current online sample mean $\hat{\mu}_i^{(\mathrm{on})}(t)$, weighted according to their respective sample sizes (Line~\ref{alg:eq:update:hybrid:mean}). The corresponding variance $\hat{\sigma}^2_{i,\mathrm{hyb}}(t)$ reflects the increased effective sample size enabled by incorporating offline data (Line~\ref{alg:eq:update:hybrid:variance}). 
Notably, to encourage sufficient exploration of the optimal arm~$1$ and reduce regret, we apply a rightward shift of magnitude $Z_{i,t}$ to the hybrid posterior distribution of each arm. A detailed discussion of the intuition behind this design choice will be provided in the end of this section. 
A hybrid posterior sample $\theta_i^{(\mathrm{hyb})}(t)$ is then drawn from the resulting Gaussian distribution (Line \ref{alg:line:theta_hybrid}).

At each time step, Anchor-TS computes three indices for every arm: the online sample mean $\hat{\mu}_i^{(\mathrm{on})}(t)$, an online posterior sample $\theta_i^{(\mathrm{on})}(t)$, and a hybrid posterior sample $\theta_i^{(\mathrm{hyb})}(t)$.
The arm score is obtained by taking the \emph{median} of these three quantities (Line \ref{alg:line:median}). And Anchor-TS selects the arm with the highest score in each round (Line \ref{alg:line:select}).

\paragraph{Intuition behind the median aggregation of three indices. }Intuitively, the role of median aggregation is to adaptively select the index that is closest to the true mean rather than to favor large or small posterior realizations like UCB-type algorithms \citep{CheungLyu2024}. This distinction is intrinsic to TS: posterior samples in TS are neither guaranteed to be optimistic nor pessimistic. Naively selecting the minimum or maximum of multiple indices would either underestimating the optimal arm or overestimating suboptimal arms, leading to uncontrolled regret. 
The sample mean $\hat{\mu}^{\mathrm{(on)}}$, estimated from online observations, therefore serves as a natural stabilizing anchor as it concentrates rapidly around the true arm mean.
By taking the median of $\hat{\mu}^{\mathrm{(on)}}$, an online-only and a hybrid posterior sample, Anchor-TS selects the index that is most consistent with this anchor.

When the bias is small, the hybrid posterior, benefiting from a larger effective sample size, tends to concentrate more tightly around the true mean than the online-only posterior.
In this regime, the median naturally favors the hybrid posterior sample, allowing Anchor-TS to exploit offline data for faster learning.
Conversely, when the offline bias is large, the hybrid posterior may deviate from the true mean, while the online-only posterior remains centered around the unbiased online signal.
In this case, the median suppresses the biased hybrid sample and tends to select the online sample instead, yielding robustness against offline-induced distortion. 
Through this adaptive selection mechanism, Anchor-TS automatically interpolates between efficiency and robustness, leveraging offline data when it is reliable and reverting to online evidence when offline bias is substantial.

\paragraph{Intuition behind the right-hand shift on the hybrid posterior distribution. } Recall that for TS-type algorithms, the regret can still accumulate when the estimates of suboptimal arms are accurate, as long as the optimal arm is poorly estimated. Specific to the offline-to-online setting, if the offline data underestimate the optimal arm, i.e., $\mu_1^{\mathrm{(off)}}<\mu_1^{\mathrm{(on)}}$, and the online sample mean of arm $1$ happens to be pessimistic, arm $1$ may receive too little posterior probability due to the three-index voting mechanism. 
In this case, it is difficult to correct the estimation error of arm $1$ through additional online observations. As a result, the persistent underestimation of arm $1$ amplifies the number of selections of suboptimal arms, leading to increased regret despite their accurate estimation. 

Our right-shift operation counteracts this effect by preventing the hybrid posterior distribution from underestimating the online mean of arm $1$. This guarantees adequate exploration of the optimal arm and prevents regret from being dominated by prolonged under-exploration of the optimal arm. 

Since the identity of the optimal arm is unknown, the right-hand shift is applied uniformly across arms. Importantly, this shift does not compromise the benefit of offline data for suboptimal arms: although overestimation by the hybrid posterior may initially increase their selection frequency, the influence of offline information quickly diminishes as online observations accumulate, and the estimate becomes dominated by online samples through median aggregation.


\section{Theoretical Results and Discussions}
\label{sec: theoretical}

The following theorem summarizes the regret bound of our algorithm:
\begin{theorem}\label{thm:main}
The cumulative regret of Algorithm \ref{alg:hyb-ts} can be bounded by
\begin{align*}
    \mathrm{Reg}(T)\le O \left( \sum_{i\neq 1} \Delta_i \left( \left( \frac{C_1\log T}{\Delta_i^2}
  - N_i\left(1 - \frac{3 \omega_i}{\Delta_i}\right)_+\right)_+   +\left(\frac{C_2\log T}{\Delta_i^2}
  - N_1 \right)_+ 
  + \frac{C_3}{\Delta_i^2} \right) \right)\,,
\end{align*}
where $C_1, C_2, C_3$ are constants, $(\cdot)_+$ represents $\max\{\cdot, 0\}$, $\omega_i := V_i+\mu_i^{\mathrm{(off)}}-\mu_i^{\mathrm{(on)}}$ represents the effective discrepancy by adding $V_i$ in the hybrid posterior distribution. 
\end{theorem}

Due to the space limit, the complete proof of Theorem \ref{thm:main} is deferred to Appendix \ref{sec:regret}. In the following, we first provide some discussions and then show the proof sketch. 

\paragraph{Intuition of the regret upper bound.}
This upper bound consists of three terms. The first term arises from inaccurate estimation of suboptimal arms, the second from inaccurate estimation of the optimal arm, and the third is a constant term required for convergence of the online sample mean.

This upper bound can be interpreted as the regret in a purely online setting minus the benefit provided by the offline data. When no offline data are available, the regret reduces to the standard purely online regret \citep{agrawal2013ts}. When the distribution shift $\mu_i^{\mathrm{(on)}}-\mu_i^{\mathrm{(off)}}$ is zero, the offline samples for both the sub-optimal arms and the optimal arm can be viewed as offsetting a portion of the regret. In particular, if the sub-optimal arms have sufficiently many offline samples, their means can be accurately estimated from the outset, rendering the corresponding regret term constant. Similarly, if the optimal arm has sufficiently many offline samples, the regret arising from inaccurate estimation of the optimal arm is likewise reduced to a constant.

Regarding the bias, for a sub-optimal arm $i$, if the offline mean is smaller than the online mean and \(V\) is a tight bound, i.e., $V = \mu_i^{\mathrm{(on)}} - \mu_i^{\mathrm{(off)}} > 0$, 
then \(\omega_i = 0\) and the offline improvement scales linearly with \(N_i\). However, if the offline mean \(\mu_i^{\mathrm{(off)}}\) is larger, the hybrid samples may overestimate the true online mean \(\mu_i^{\mathrm{(on)}}\). In this case, the contribution of the offline data is attenuated by a discount factor \((1 - 3\omega_i / \Delta_i)\), which decreases as \(\omega_i\) increases, resulting in diminishing effective information from the offline samples. When \(3\omega_i > \Delta_i\), the hybrid samples can mislead the identification of the optimal arm, and as the result the offline benefit vanishes and the regret reduces to that of the purely online setting. 
In contrast, for the optimal arm (arm~1), adding the correction term \(V\) to the hybrid samples always leads to an overestimation of the true mean \(\mu_1^{\mathrm{(on)}}\). Such overestimation consistently favors the identification of the optimal arm. Consequently, the contribution of the offline data for the optimal arm is not subject to any discount factor.


\paragraph{Comparison with \citet{CheungLyu2024}. }
\citet{CheungLyu2024} study the same offline-to-online setting as ours, but focus on a UCB-based approach. They derive an upper bound on the sub-optimal $i$'s selection time $O({C \log T}/{\Delta^2_i}-N_i \cdot \max\left\{1-{\omega_i}/{\Delta_i}, 0  \right\}^2 
)_+ $. 
The first term in our Theorem~\ref{thm:main}, which captures the regret due to inaccurate estimation of suboptimal arms, is of the same order as this bound. The main difference lies in how the size of the offline dataset $N_i$ is discounted. Specifically, their bound uses the discount factor $(1-\omega_i/\Delta_i)^2$ whereas ours involves the factor $(1-3\omega_i/\Delta_i)$. This constant coefficient discrepancy primarily stems from differences in the analysis techniques. In TS, posterior samples may be either optimistic or pessimistic, which requires partitioning the gap between $\mu_1$ and $\mu_i$ into three regions, leading to the $\Delta_i/3$ factor in our formula. 
Importantly, our analysis avoids an additional squaring of the discount factor.

Compared with this result, our Theorem~\ref{thm:main} further reveals an additional improvement arising from the offline data $N_1$ on the optimal arm, as reflected in the second term of our bound. This improvement stems from the property that TS incurs regret due to inaccurate estimation of both suboptimal and optimal arms. When offline data are available for arm $1$, the estimation of the optimal arm becomes more accurate, thereby reducing this source of regret. 
Such a scenario is common in practical applications, since offline data are typically collected using expert or near-optimal policies, and therefore observations of the optimal arm are often abundant. Our Anchor-TS algorithm is able to exploit this advantage, whereas UCB-based methods do not benefit from offline data on the optimal arm in the same way.

Another aspect worth discussion is the right-shift operation applied to the hybrid samples. Although both the MINUCB algorithm and our approach incorporate a bias-correction term into the hybrid index, the underlying motivations are fundamentally different. In MINUCB, the correction term is introduced to preserve the optimism of the hybrid UCB index, ensuring that it upper-bounds the true reward.
In contrast, TS does not intrinsically rely on optimism. The primary purpose of introducing the right-shift in our method is to encourage exploration of arm $1$, thereby mitigating the regret caused by underestimation of this arm. While omitting this correction term would in fact improve the first regret term from over-estimation of sub-optimal arms, it would significantly complicate the analysis of $\hat{\mu}_1^{\mathrm{(on)}}$. In particular, inaccuracies in estimating $\hat{\mu}_1^{\mathrm{(on)}}$ would introduce regret terms that grow exponentially with $N_1$, since underestimation of the hybrid sample $\theta_1^{\mathrm{(hyb)}}$ would make arm~1 increasingly unlikely to be selected, leaving little opportunity for correction.


\paragraph{Reduction to the pure online setting and the role of median aggregation. } In the pure online setting, where no offline data are available, the hybrid posterior distribution coincides with the online posterior distribution. In this case, Anchor-TS can be interpreted as sampling two online indices, $\theta_{i,1}^{\mathrm{(on)}}$, $\theta_{i,2}^{\mathrm{(on)}}$, and selecting the median among these two samples and the online sample mean $\hat{\mu}_i^{\mathrm{(on)}}$. We analyze this setting separately and show that this modification preserves the classical instance-dependent regret bound, while improving the leading regret order by a factor of $1/2$. We further provide empirical results that illustrate this advantage. The corresponding theoretical analysis and experimental results are presented in Appendix~\ref{appendix: discussion}. 
This observation may, to some extent, reflect a similar idea in \citet{jin2023}, who show that TS with less exploration can perform better. But we leverage different ideas to reduce exploration: they divert a fixed probability mass $\epsilon$ from vanilla TS to selecting the sample mean, whereas our method incorporates the sample mean in a more adaptive manner through median aggregation.



\section{Proof Sketch}

In this section, we present a proof sketch for Theorem~\ref{thm:main} and highlight the key ideas of our analysis. Our technical contributions operate at two levels: at a high level, we tightly couple the median-based algorithmic design with the regret decomposition, enabling the regret bound to adapt to the more informative of the online and hybrid posterior distributions; at a more technical level, we develop a refined probabilistic analysis that decouples multiple random indices appearing within probability and expectation operators, allowing their individual regret contributions to be explicitly controlled.

Standard TS analyses consider regret contributions from a single posterior distribution. In contrast, our setting involves two posterior sources, where a naive extension yields additive regret bounds and fails to exploit the more informative one. To address this, we use the online sample mean as an anchor: its estimation errors contribute only a constant-order regret of $O(1/\Delta_i)$, allowing us to condition on accurate estimation and isolate the effects of the two posterior distributions.




For convenience, define $x_i := \mu^{\mathrm{(on)}}_i + {\Delta_i}/{3}$ and $y_i := \mu^{\mathrm{(on)}}_i + {2\Delta_i}/{3}$ as two thresholds, $E_i^{\mu(\mathrm{on})}(t):= \{ \hat\mu^{\mathrm{(on)}}_i(t) \le x_i \}$ as the good event that arm $i$'s online sample mean is accurate. 
Then, 
\begin{align*}
    \E \left[\sum_{t=1}^T\bOne{ A(t)=i }\right]  = & \sum_{t=1}^T \Pr(A(t)=i,\; \neg E_i^{\mu(\mathrm{on})}(t)) + \sum_{t=1}^T \Pr(A(t)=i,\; E_i^{\mu(\mathrm{on})}(t)) \,.
\end{align*}
The first term can be upper bounded by $O(1/\Delta_i^2)$ without relying on $T$. The second term representing regret from selecting $i$ when its online sample mean is accurate can be decomposed as 
\begin{align}
    \sum_{t=1}^T \Pr(A(t)=i, E_i^{\mu(\mathrm{on})}(t), \hat{\theta}_i(t)>y_i )+ \sum_{t=1}^T \Pr( A(t)=i, E_i^{\mu(\mathrm{on})}(t), \hat{\theta}_i(t)\le y_i )\,. \label{eq:main:sketch:two-terms}
\end{align}
The first term in \eqref{eq:main:sketch:two-terms} corresponds to the case where arm $i$'s median index is over-estimated. 
Conditional on the good event $\{ \hat{\mu}_i^{\mathrm{(on)}}(t) \le x_i\}$, the event $\{\hat{\theta}_i(t)>y_i \}$ implies that $\{{\theta}^{\mathrm{(on)}}_i(t)>y_i \} \cap \{{\theta}^{\mathrm{(hyb)}}_i(t)>y_i \} $. The corresponding regret event is therefore the intersection of two unfavorable events, whose probability is strictly smaller than that induced by either posterior alone,  
\begin{align*}
   \sum_{t=1}^T \min \left\{  
     \Pr(A(t)=i, E_i^{\mu(\mathrm{on})}(t), {\theta}_i^{\mathrm{(on)}}(t)>y_i   ), 
     \Pr(A(t)=i, E_i^{\mu(\mathrm{on})}(t), {\theta}_i^{\mathrm{(hyb)}}(t)>y_i   )
    \right\} \,.
\end{align*}
This yields the first term in Theorem~\ref{thm:main} by taking the minimum of the corresponding regret contributions from the online and hybrid posterior distributions.

The second term in \eqref{eq:main:sketch:two-terms} can be transformed to the case where arm $1$'s median index is under-estimated. 
Dealing with this term is commonly recognized as a key challenge in TS analyses \citep{agrawal2013ts,jin2021mots,jin2023}, which is typically upper bounded by $\E[1/p-1]$ where $p$ is the conditional probability (given the history) that $\hat{\theta}_1>y_i$. In vanilla TS, $p$ is the tail probability of a single posterior sample $\hat{\theta}_1:=\theta_1^{\mathrm{(on)}}$, so one can directly convert $\E[1/p]$ into the expectation of a geometric hitting time related to the behavior of the index. In our setting, however, the arm index is the median of three random quantities with distinct concentration behaviors, which prevents a direct reduction to a one-dimensional geometric hitting-time argument as in vanilla TS.

To handle this term, we further condition on the behavior of $\hat{\mu}_1^{\mathrm{(on)}}$. Under the good event $\{\hat{\mu}_1^{\mathrm{(on)}}(t) > y_i\}$, 
it suffices that either the online or the hybrid posterior sample of arm $1$ exceeds $y_i$ for the optimal arm to be selected. This yields a union of favorable events, implying that
$$p \ge \max\{p^{\mathrm{(on)}}, p^{\mathrm{(hyb)}}\}, \text{ thus  } 
{1}/{p}-1 \le \min \{{1}/{p^{\mathrm{(on)}}}-1,\; {1}/{p^{\mathrm{(hyb)}}}-1 \}\,,$$
where $p^{\mathrm{(on)}}$ and $p^{\mathrm{(hyb)}}$ represent the conditional probability that $\{{\theta}_1^{\mathrm{(on)}}(t)>y_i \}$ and $\{{\theta}_1^{\mathrm{(hyb)}}(t)>y_i \}$, respectively. This leads to the second term in Theorem~\ref{thm:main}, where the right-hand shift of the hybrid posterior distribution ensures that $p^{\mathrm{(hyb)}}>p^{\mathrm{(on)}}$ when offline data arm $1$ are available.



Another key difficulty arises when $\hat{\mu}_1^{\mathrm{(on)}}(t)<y_i$. In this regime, the median index can no longer exploit the more favorable of the two posterior samples, and the success probability $p$ reduces to an intersection event involving both samples. This joint event can significantly shrink $p$ and amplify regret. Moreover, since $\hat{\mu}_1^{\mathrm{(on)}}(t)$ and $\hat{\theta}_1(t)$ are statistically coupled, their contributions cannot be bounded separately and combined multiplicatively.

To address this challenge, we first apply Hölder’s inequality to decompose the regret term into two components that can be handled separately. The first component involves the success probability $1/p$ corresponding to a joint hitting event of two posterior samples. Here, the right-hand shift of the hybrid posterior distribution plays a crucial role by preventing this term from becoming excessively large, ensuring that it remains bounded by a constant. The second component captures the probability of inaccurate estimation of the online sample mean, which decreases exponentially as the selection of arm $1$. As a result, the combined contribution of these two components can be controlled by an $O(1/\Delta^2_i)$ upper bound.

\section{Experiments} \label{sec: experiments}
In this section, we compare our Anchor-TS with three classes of baselines: online-only methods including standard TS~\citep{agrawal2013ts} and standard UCB~\citep{Auer2002}; naive offline-to-online methods including hybrid TS and hybrid UCB \citep{shivaswamy2012multi} which use offline data for initialization while treating it as unbiased; and bias-aware offline-to-online method MINUCB \citep{CheungLyu2024}. For each experiment, we report cumulative regret over $10k$ rounds, averaged over $50$ runs with error bars showing the standard error. 
We consider a basic setting with $K=10$ arms, where the optimal arm has online mean reward $0.8$ and all suboptimal arms have mean $0.5$, yielding a gap of $0.3$. Rewards are drawn from a Gaussian distribution with unit variance.
The total number of offline samples is $2k$. To investigate the effect of offline data coverage, we consider three coverage patterns: uniform coverage across arms, coverage concentrated on the optimal arm $1$ (with $80\%$ of samples), and coverage concentrated on a suboptimal arm $2$ (with $80\%$ of samples). We further explore algorithm performance by varying parameters of this basic setting.

\paragraph{Unbiased offline data.}
We first consider a simple setting with zero bias under different sub-optimality gaps $\Delta\in \{0.3, 0.1\}$. 
The results are reported in Figure \ref{fig:fig1}.
There are two common observations: methods that leverage offline data consistently outperform purely online methods, and TS-type algorithms typically outperform their UCB-type counterparts.

In the easiest setting with uniformly distributed and sufficiently abundant offline data (Figure~\ref{fig:fig1} left(a)), both hybrid UCB and hybrid TS achieve zero regret. In contrast, Anchor-TS incurs a small constant regret, as it deliberately requires a constant number of online samples to form a reliable anchor. 
In other settings, Anchor-TS consistently achieves the lowest regret and outperforms hybrid TS, despite both methods leveraging the same offline information. This gap comes from different exploration control: hybrid TS relies on a single high-variance posterior sample and may tend to over-explore, whereas Anchor-TS reduces exploration through sample-mean–based anchoring. A similar phenomenon was also observed in \citet{jin2023}, where incorporating the sample mean into the decision process leads to improved performance. 

Compared with UCB-based algorithms, Anchor-TS consistently achieves substantial gains, particularly when offline data are concentrated on the optimal arm. This behavior aligns with our theoretical analysis, which predicts that the benefit of Anchor-TS scales with the amount of offline data available for the optimal arm, while UCB-based methods do not benefit in the same way. This advantage becomes even more pronounced in harder settings with a smaller gap where hybrid UCB and MINUCB even underperform pure online UCB (Figure \ref{fig:fig1} right (b)). This is because the UCB of the optimal arm is excessively reduced when incorporating offline data, forcing other arms to be pulled more often before their sub-optimality can be identified.

\begin{figure}[t]
\centering
\begin{minipage}[t]{0.496\textwidth}
  \centering
  \includegraphics[width=\linewidth]{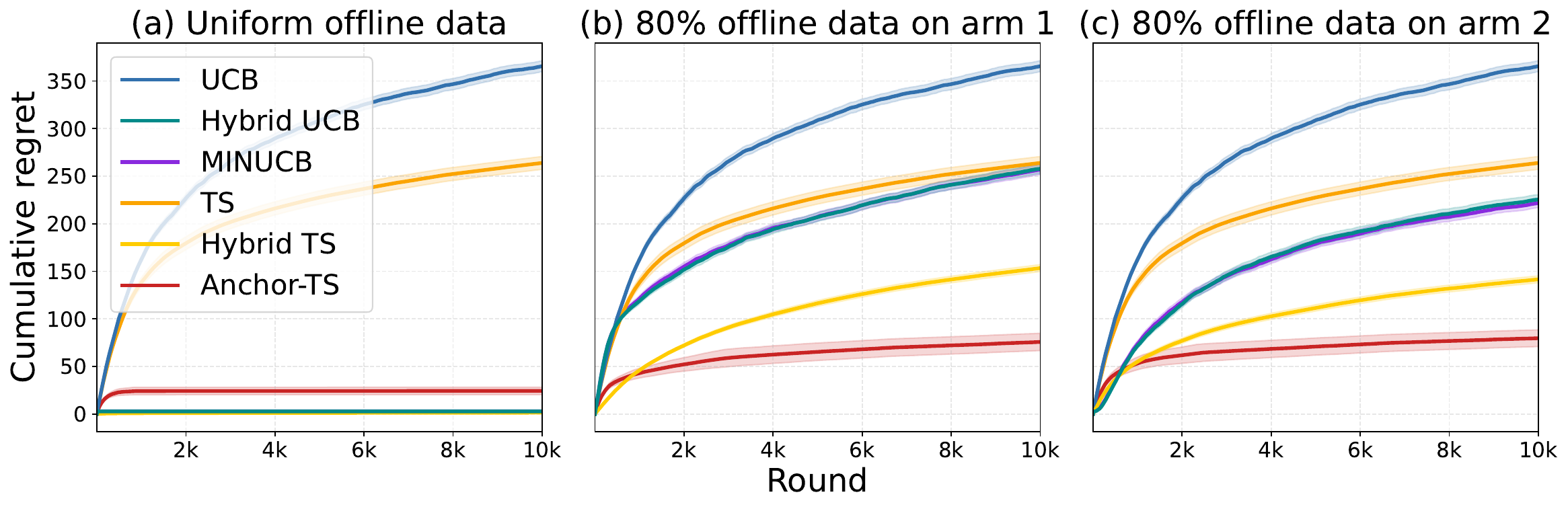}
\end{minipage}%
\hspace{0.3em}%
\begin{minipage}[t]{0.496\textwidth}
  \centering
  \includegraphics[width=\linewidth]{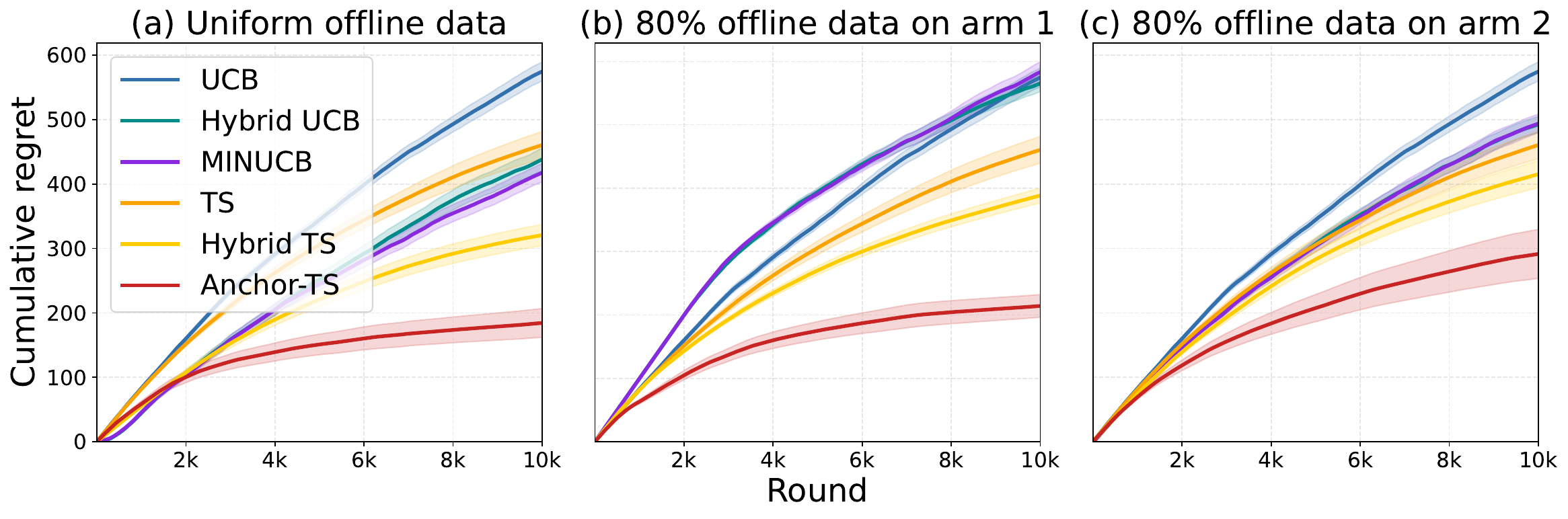}
\end{minipage}
\caption{Cumulative regret in the unbiased setting under varying offline coverage regimes and suboptimality gaps (left: $\Delta = 0.3$; right: $\Delta = 0.1$). }
\label{fig:fig1}
\end{figure}

\paragraph{Biased offline data.}We then evaluate the performance of algorithms under distribution shift between offline and online environments. We include both pure online methods and bias-aware approaches as baselines. To ensure that the offline data is sufficiently misleading, we consider a base setting in which $\mu_1^{\mathrm{(off)}} = 0.5$ and $\mu_i^{\mathrm{(off)}} = 0.6$ for $i \neq 1$, so that the optimal arm is underestimated in the offline data while the suboptimal arms are overestimated. Figure~\ref{fig:fig2} summarizes the algorithms performances under biased offline data by varying key problem parameters, including the total offline sample size, the real bias level, the hyper-parameter $V$,  and the arm set size. Unless otherwise specified, the hyperparameter $V$ is  set to the true bias magnitude for each arm. Across all settings, Anchor-TS consistently achieves the lowest regret, demonstrating robustness to offline bias and markedly outperforming UCB-based methods in its ability to effectively utilize offline information.

\captionsetup[figure]{skip=2pt}
\begin{figure*}[t]
\centering

\begin{minipage}[t]{0.491\textwidth}
  \centering
  \hspace*{-8pt}
  \includegraphics[width=\linewidth]{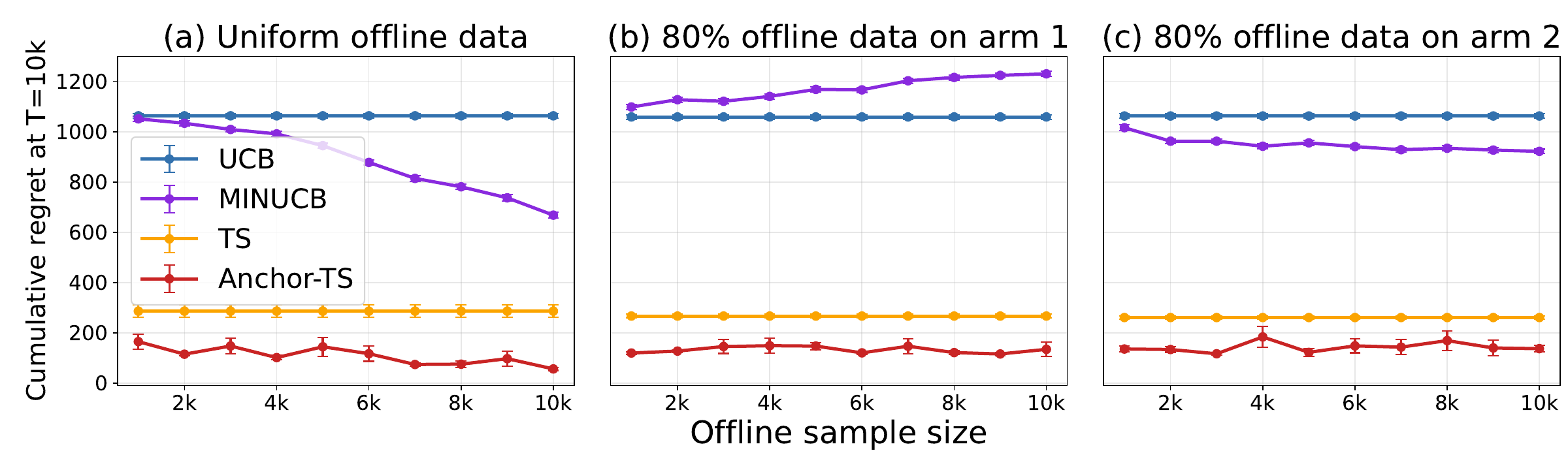}
\end{minipage}\hspace{0.01\textwidth}%
\begin{minipage}[t]{0.493\textwidth}
  \centering
  \includegraphics[width=\linewidth]{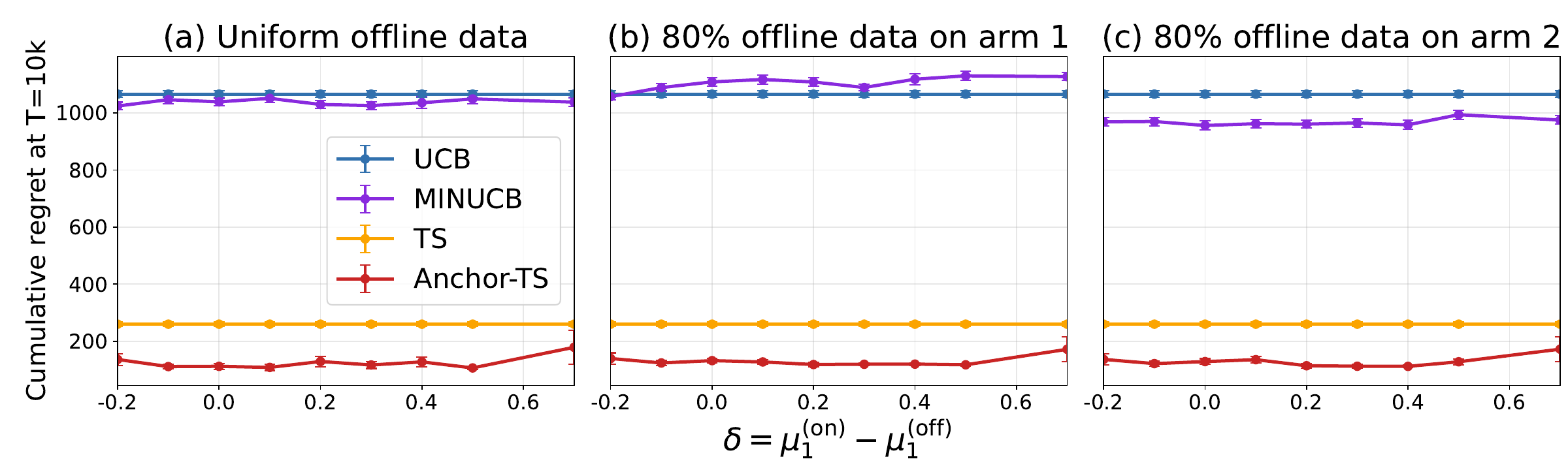}
\end{minipage}

\vspace{2pt}

\begin{minipage}[t]{0.493\textwidth}
  \centering
  \hspace*{-8pt}
  \includegraphics[width=\linewidth]{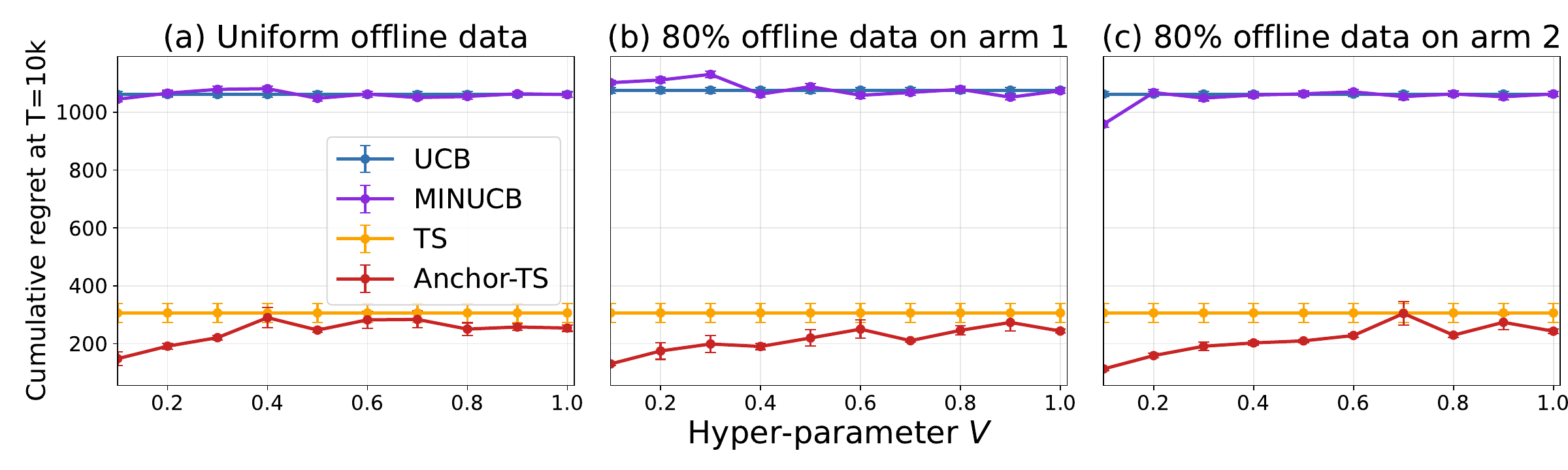}
\end{minipage}\hspace{0.01\textwidth}%
\begin{minipage}[t]{0.493\textwidth}
  \centering
  \includegraphics[width=\linewidth]{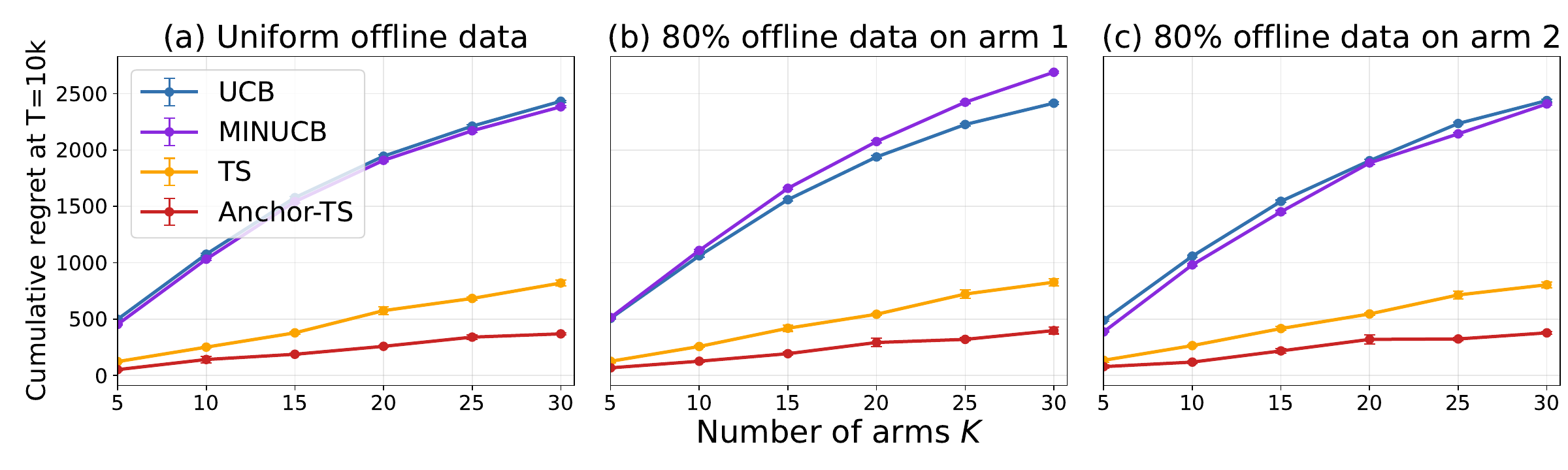}
\end{minipage}

\caption{Cumulative regret under varying offline coverage regimes and problem parameters
(top-left: total offline sample size; top-right: bias level $\delta$; bottom-left: hyperparameter $V$; bottom-right: number of arms $K$).}
\label{fig:fig2}
\end{figure*}

The top-left panel shows the algorithms performance under different offline sample sizes $\sum_i N_i$. Increasing offline data generally improves bias-aware methods, especially under uniform coverage.

The top-right panel studies the effect of the bias level $\delta$. 
For ease of implementation, we fix the offline mean of sub-optimal arms and vary $\mu_1^{\mathrm{(off)}}$ by setting different choices of $\delta := \mu_1^{\mathrm{(on)}} - \mu_1^{\mathrm{(off)}}$. Larger $\delta$ corresponds to stronger underestimation of the optimal arm in the offline data. Nevertheless, Anchor-TS maintains a stable low-regret profile and consistently outperforms baselines.

The bottom-left panel examines the impact of hyper-parameter $V$. For MINUCB and Anchor-TS, we set $V$ to the maximum of the true bias magnitude and the tested value. When $V$ is small, tighter prior control allows Anchor-TS to exploit offline data more aggressively. As $V$ increases, its performance degrades but never falls below pure online TS. Even when offline data are uninformative with large $V$, Anchor-TS retains a slight advantage due to sample-mean based exploration. 

The bottom-right panel evaluates scalability with respect to the number of arms $K$. While regret increases with $K$ for all methods, UCB-based algorithms grow much faster, whereas Anchor-TS scales more favorably by suppressing unnecessary exploration.

\section{Conclusion}
In this work, we study the TS-type algorithm for offline-to-online learning. Such algorithms face a fundamental challenge in assessing the reliability of the hybrid posterior distribution through comparison with the online posterior, stemming from the lack of inherent optimism in posterior samples. We address this challenge by introducing a novel median-based aggregation mechanism that combines the online sample mean, an online posterior sample, and a hybrid posterior sample, enabling adaptive exploitation of informative offline data. Our theoretical results demonstrate that our algorithm can effectively cope with bias in offline data, while achieving strictly better regret improvement when the offline data are sufficiently informative. Compared with UCB-type algorithms, our TS-based approach can further leverage abundant offline data on the optimal arm to obtain larger performance gains. Moreover, even in the pure online setting, the proposed median-index mechanism yields provable improvement in the leading constant of the regret bound compared with pure online TS. Extensive empirical evaluations further validate our theoretical findings, demonstrating that our algorithm consistently outperforms state-of-the-art baselines across various settings of data coverage and bias magnitudes. 

Several promising directions merit further investigation. First, it would be valuable to establish gap-independent regret bounds that explicitly quantify the benefits of offline data. Second, extending the proposed anchor-based framework to contextual TS is an important direction for addressing high-dimensional, real-world decision-making problems.



\bibliography{Reference}
\bibliographystyle{named}

\appendix

\newpage

\section{Proof of Theorem \ref{thm:main}}
\label{sec:regret}

We first introduce some useful definitions that will be used in the proof.

For each sub-optimal arm $i\neq 1$, we define two thresholds: $x_i := \mu^{\mathrm{(on)}}_i + {\Delta_i}/{3}$ and $y_i := \mu^{\mathrm{(on)}}_i + {2\Delta_i}/{3}
       = \mu^{\mathrm{(on)}}_1 - {\Delta_i}/{3}$. It holds that $x_i<y_i$ and $y_i-x_i = \Delta_i/3$.

 Define the confidence radius $L_t := 2\log(2t/\delta_t)$ with uncertainty level $\delta_t= 1/ (t^2 \Delta_i^2)$. 

Let $\mathcal{F}_t:= \cup_{i=1}^K S_i \cup\{(A(\tau), R_{A(\tau)}(\tau)):1\le \tau \le t \}$ to represent the history comprising the offline dataset and the sequence of online observations collected up to round $t$. 

For the optimal arm $1$ and arm $i$, define $p_{i,t}
:=\Pr\bigl(\hat\theta_1(t)>y_i\mid\mathcal F_{t-1}\bigr)$ to represent the conditional probability that the median-index is larger than $y_i$. Similarly, define $p^{(\mathrm{on})}_{i,t}
:=\Pr\bigl(\theta^{(\mathrm{on})}_1(t)>y_i\mid\mathcal F_{t-1}\bigr)$ and $p^{(\mathrm{hyb})}_{i,t}
:=\Pr\bigl(\theta^{(\mathrm{hyb})}_1(t)>y_i\mid\mathcal F_{t-1}\bigr)$. 

Further define the following good events representing the estimators are not far from their centers. 

\begin{definition}[Good events $E_i^{\mu(\mathrm{on})}(t), E_i^{\mu(\mathrm{hyb})}(t), E_i^{\theta(\mathrm{hyb})}(t), \mathcal{G}_i(t), E_1^{\mu(\mathrm{on})}(t)$]\label{def: good events}
 For $i\ne 1$, define the following good event 
 \begin{align}
  E_i^{\mu(\mathrm{on})}(t)
  &:= \Bigl\{
       \hat\mu^{\mathrm{(on)}}_i(t) \le x_i
     \Bigr\}\,, \nonumber\\
  E_i^{\mu(\mathrm{hyb})}(t)
  &:= \Biggl\{
       \bigl|\hat\mu^{\mathrm{(hyb)}}_i(t)
             - \mu^{\mathrm{(on)}}_i\bigr|
       \le
       \sqrt{\frac{L_t}{N_i(t)+T_i}}
       + \frac{T_i V_i}{N_i(t)+T_i}
     \Biggr\}\,, \nonumber
     \\
  E_i^{\theta(\mathrm{hyb})}(t)
  &:= \Biggl\{
       \bigl|\theta_i^{(\mathrm{hyb})}(t) - \hat\mu^{\mathrm{(hyb)}}_i(t)\bigr|
       \le
       \sqrt{\frac{L_t}{N_i(t)+T_i}}+\frac{T_i V_i}{N_i(t)+T_i}
     \Biggr\} \,. \nonumber
\end{align}
and the global good event
\begin{align}
    \mathcal{G}_i(t)
  := E_i^{\mu(\mathrm{on})}(t)
     \cap E_i^{\mu(\mathrm{hyb})}(t)
     \cap E_i^{\theta(\mathrm{hyb})}(t)\,. \nonumber
\end{align}
Similarly, for optimal arm $1$, define 
\begin{align}
    E_1^{\mu(\mathrm{on})}(t):=\bigl\{\hat\mu^{(\mathrm{on})}_1(t)>y_i\bigr\}\,. \nonumber
\end{align}
\end{definition}
 

In the following, we provide the detailed proof of Theorem \ref{thm:main}. 

\begin{proof}[Proof of Theorem \ref{thm:main}]
    
We first analyze the regret by bounding the number of selections of each sub-optimal arm $i\neq 1$.
{\allowdisplaybreaks
\begin{align}
    \E \left[\sum_{t=1}^T\bOne{ A(t)=i }\right]  = & \sum_{t=1}^T \Pr(A(t)=i,\; \neg E_i^{\mu(\mathrm{on})}(t)) + \sum_{t=1}^T \Pr(A(t)=i,\;\neg E_i^{\mu(\mathrm{hyb})}(t))\nonumber \\
    & + \sum_{t=1}^T \Pr(A(t)=i,\;                \neg E_i^{\theta(\mathrm{hyb})}(t))+ \sum_{t=1}^T \Pr(A(t)=i, \mathcal{G}_i(t))\label{eq:NiT-4parts}  \\
    \le &  O\left(\frac{1}{\Delta_i^2}\right) + \sum_{t=1}^T \Pr(A(t)=i, \mathcal{G}_i(t))\label{eq:due:bad:events} \\
    \le & O\left(\frac{1}{\Delta_i^2}\right)  + \sum_{t=1}^T \Pr\underbrace{(A(t)=i, \mathcal{G}_i(t), \hat{\theta}_i(t)>y_i )}_{\mathcal{E}_{1,t}} \nonumber \\
    & \quad \quad \quad \quad \quad \quad \quad + \sum_{t=1}^T \Pr\underbrace{( A(t)=i, \mathcal{G}_i(t), \hat{\theta}_i(t)\le y_i )}_{\mathcal{E}_{2,t} }
    \nonumber \\
    \le & O\left(\frac{1}{\Delta_i^2}\right)  + \sum_{t=1}^T \Pr(\mathcal{E}_{1,t}) + \sum_{k=0}^{T-1}
\mathbb E\!\left[
  \frac{1-p_{i,\tau_k+1}}{p_{i,\tau_k+1}}\,\bOne{E_1^{\mu(\mathrm{on})}(\tau_k+1)}
\right]  \notag \\
&+ \sum_{k=0}^{T-1}
\mathbb E\!\left[
  \frac{1-p_{i,\tau_k+1}}{p_{i,\tau_k+1}}\,\bOne{ \neg E_1^{\mu(\mathrm{on})}(\tau_k+1)}
\right]  \label{eq:E1andE2bound} \\
    \le & \frac{C_3}{\Delta_i^2}  + \left(
  \frac{C_1 \log T }{\Delta_i^2}
  - N_i\max\left\{\Bigl(1 - \frac{3\omega_i}{\Delta_i}\Bigr)
  ,0 \right\}\right)_+ + c\left(\frac{C_2\log T}{\Delta_i^2}
- N_1\right)_+  \,. \label{proof:final} 
\end{align}}
Here \eqref{eq:due:bad:events} follows from Lemma~\ref{lemma: vanilla-bad-event-mu}, Lemma \ref{lem:term2-bound}, and Lemma \ref{lemma:term3-bound} for bounding the bad events. \eqref{eq:E1andE2bound} is due to Lemma \ref{lemma: median-trans}. \eqref{proof:final} comes from the upper bound for $\mathcal{E}_{1}$ in Lemma \ref{lem:goodEvent:sub-optimalArm} and the upper bound for two terms of $\mathcal{E}_{2}$ in  Lemma \ref{lemma: optim-good-inter} and Lemma \ref{lemma: optim-bad-term}. $C_1, C_2, C_3$ are constant terms independent of the problem parameters, $c = \max\left\{e^{11}, \exp\left\{\tfrac{28+16\sqrt{3}}{N_1+1}\right\}\right\}+5$ is a constant smaller than the coeffient in the pure online TS \citep{agrawal2013ts} and decreases as $N_1$ increases. 
$O$ hides constant terms.

The final regret can then be obtained by 
\begin{align*}
    \mathrm{Reg}(T) = & \sum_{i\neq 1} \Delta_i \mathbb{E} \left[ \sum_{t=1}^T \bOne{A(t)=i} \right] \\
    \le &  \sum_{i\neq 1} \Delta_i \left(  \frac{C_3}{\Delta_i^2} + \left(
  \frac{C_1 \log T }{\Delta_i^2}
  - N_i\max\left\{\Bigl(1 - \frac{3\omega_i}{\Delta_i}\Bigr)
  ,0 \right\}\right)_+ + c\left(\frac{C_2\log T}{\Delta_i^2}
- N_1\right)_+ \right) \,.
\end{align*}

\end{proof}

We next introduce the lemmas used in the above proof and provide the proof of lemmas in Appendix \ref{appendix:proof:lemmas}. The first two lemmas establish constant upper bounds on the probability of bad events.

\begin{lemma}\label{lem:term2-bound}
    For sub-optimal arm $i\ne 1$, the second term in \eqref{eq:NiT-4parts} can be bounded by
    $$
    \sum_{t=1}^T \Pr(A(t)=i,\;\neg E_i^{\mu(\mathrm{hyb})}(t))\le\sum_{t=1}^{T} \delta_{t}= O\left(\frac{1}{\Delta_i^2}\right).
    $$
\end{lemma}

\begin{lemma}\label{lemma:term3-bound}
For sub-optimal arm $i\ne 1$, the third term in \eqref{eq:NiT-4parts} can be bounded by
    $$
    \sum_{t=1}^T \Pr(A(t)=i,\;                \neg 
    E_i^{\theta(\mathrm{hyb})}(t))\le \sum_{t=1}^T\frac{\delta_t}{t} =O\left(\frac{1}{\Delta_i^2}\right)\,.
    $$
\end{lemma}

The following lemma establishes an upper bound on $\sum_{t=1}^T \Pr(\mathcal{E}_{1,t})$ in the regret, corresponding to the regret incurred when the global good event $\mathcal{G}_i(t)$ holds and the index $\hat{\theta}_i(t)$ of arm $i$ is inaccurate.

\begin{lemma}\label{lem:goodEvent:sub-optimalArm}
For sub-optimal arm $i\ne 1$,
    $$
    \sum_{t=1}^T \Pr(\mathcal{E}_{1,t})  \le\left(
  36\frac{\log (T\Delta_i^2+e^{6})}{\Delta_i^2}
  - N_i\max\left\{\Bigl(1 - \frac{3\omega_i}{\Delta_i}\Bigr)
  ,0 \right\}\right)_++ \frac{1}{\Delta_i^2}.
    $$
\end{lemma}

The following lemma establishes an upper bound on $\sum_{t=1}^T \Pr(\mathcal{E}_{2,t})$ in the regret. 
This term can be transformed into the case where the index $\hat{\theta}_1(t)$ of arm $1$ is inaccurate, which may lead to the selection of arm $i$.

\begin{lemma}\label{lemma: median-trans}
    For all $t$, $i\ne 1$ and all instantiations $F_{t-1}$ of $\mathcal{F}_{t-1}$, 
\begin{equation}
\Pr\bigl(A(t)=i,\;\mathcal G_i(t), \; \hat{\theta}_i(t)\le y_i \mid F_{t-1}\bigr)
\;\le\;
\frac{1-p_{i,t}}{p_{i,t}}\,
\Pr\bigl(A(t)=1,\;\mathcal G_i(t),\; \hat{\theta}_i(t)\le y_i \mid F_{t-1}\bigr).\nonumber 
\end{equation}
Further,  
$$
\sum_{t=1}^T\Pr\bigl(\mathcal{E}_{2,t}\bigr) \le \sum_{k=0}^{T-1}
\mathbb E\!\left[
  \frac{1-p_{i,\tau_k+1}}{p_{i,\tau_k+1}}\,\bOne{E_1^{\mu(\mathrm{on})}(\tau_k+1)}
\right]+\sum_{k=0}^{T-1}
\mathbb E\!\left[
  \frac{1-p_{i,\tau_k+1}}{p_{i,\tau_k+1}}\,\bOne{ \neg E_1^{\mu(\mathrm{on})}(\tau_k+1)}
\right].
$$
\end{lemma}

\begin{lemma}\label{lemma: optim-good-inter}
For any $t\ge 0$, 
    \begin{equation}
\sum_{k=0}^{T-1}
\mathbb E\!\left[
  \frac{1-p_{i,\tau_k+1}}{p_{i,\tau_k+1}}\,\bOne{E_1^{\mu(\mathrm{on})}(\tau_k+1)}
\right]
\;\le\;
\sum_{k=0}^{T-1}
\min\Bigl\{
  \mathbb E\bigl[\tfrac1{p^{(\mathrm{on})}_{i,\tau_k+1}}-1\bigr],\;
  \mathbb E\bigl[\tfrac1{p^{(\mathrm{hyb})}_{i,\tau_k+1}}-1\bigr]
\Bigr\}.\nonumber
\end{equation}
\end{lemma}

 We then prove a similar lemma corresponding to Lemma~\ref{lemma: vanilla optim-term} for the hybrid estimator.

\begin{lemma}\label{lemma:hybrid bound-optimal}
Let $\tau_k$ be the time of the $k$-th play of arm~1. 
Then
\[
\mathbb E\Big[\tfrac1{p^{(\hyb)}_{i,\tau_k+1}}-1\Big]
\le
\begin{cases}
\max\left\{e^{11}, \exp\left(\frac{28+16\sqrt{3}}{N_1+1}\right)\right\} + 5 , &  \forall\; k,\\
\dfrac{5}{T\Delta_i^2}, & k> L_i^{\mathrm{(hyb)}}(T),
\end{cases}
\]
where $L_i^{\mathrm{(hyb)}}(T) = \left(\frac{288\ln(T\Delta_i^2 + e^{6})}{\Delta_i^2} -N_1\right)_+.$
\end{lemma}


\begin{lemma}\label{lemma:optimal-good-part}
    \begin{equation}
        \sum_{k=0}^{T-1}
\mathbb E\!\left[
  \frac{1-p_{i,\tau_k+1}}{p_{i,\tau_k+1}}\,\bOne{E_1^{\mu(\mathrm{on})}(\tau_k+1)}
\right]\;\le\;
C_1\left(\,\frac{\log (T\Delta_i^2+e^{32})}{\Delta_i^2}
- N_1\right)_+ + C_2\,\frac1{\Delta_i^2},\nonumber
    \end{equation}
for suitable constants $C_1,C_2>0$ independent of $T$. 
\end{lemma}

\begin{proof}[Proof of Lemma \ref{lemma:optimal-good-part}]
    Combining Lemma~\ref{lemma: optim-good-inter}, Lemma~\ref{lemma: vanilla optim-term} and Lemma~\ref{lemma:hybrid bound-optimal} we get the bound. 
\end{proof}



\begin{lemma}\label{lemma: optim-bad-term}
    \begin{equation}
        \sum_{k=0}^{T-1}
\mathbb E\!\left[
  \frac{1-p_{i,\tau_k+1}}{p_{i,\tau_k+1}}\,\bOne{\neg E_1^{\mu(\mathrm{on})}(\tau_k+1)}
\right]\;\le\;\frac{C}{\Delta_i^2}
,\nonumber
    \end{equation}
for some constant $C$.    
\end{lemma}


\section{Theoretical Analysis and Experiments in the Pure Online Setting} \label{appendix: discussion}
In this section, we provide a tighter analysis to the pure online setting, where no offline data are available. In this case, the hybrid posterior coincides with the online posterior, and Anchor-TS reduces to selecting arms according to the median index
\begin{equation}\nonumber
    \mathrm{median}\left\{\theta_{i,1}^{\mathrm{(on)}}(t),\theta_{i,2}^{\mathrm{(on)}}(t),\hat{\mu}_i^{\mathrm{(on)}}\right\}\,,
\end{equation}
where $\theta_{i,1}^{\mathrm{(on)}}$ and $\theta_{i,2}^{\mathrm{(on)}}$ are i.i.d. samples from the standard online posterior of arm $i$ at time $t$. We show that this median rule preserves the classical instance-dependent regret guarantee while improving the leading logarithmic term by a factor $1/2$ compared with vanilla TS. 

\subsection{ $1/2$ improvement on the leading $\log T$-term}
We reuse the notation in Appendix~\ref{sec:regret}. Recall that the $\log T$ contribution in Theorem~\ref{thm:main} arises from two parts: (i) controlling over-optimistic indices of suboptimal arms (captured by the term $\sum_{t=1}^T \Pr(\mathcal{E}_{1,t})$ as in Lemma~\ref{lem:goodEvent:sub-optimalArm}), and (ii) controlling under-exploration of the optimal arm (via the ratio term in Lemma~\ref{lemma: median-trans}). We bound these two parts separately.
\paragraph{Suboptimal-arm term. }
Recall that for a suboptimal arm $i\ne 1$, 
\begin{equation}\label{eq: discuss-sub-orin}
    \sum_{t=1}^T \Pr(\mathcal{E}_{1,t}) = \sum_{t=1}^T \Pr\left(A(t)=i, \mathcal{G}_i(t), \hat{\theta}_i(t)>y_i\right)\,.
\end{equation}

Conditioned on the global good event $\mathcal{G}_i(t)$, we have $\hat{\mu}_i^{\mathrm{(on)}}< x_i<y_i$. Therefore, the event $\{\hat{\theta}_i(t)>y_i\}$ can only happen when both posterior samples exceed $y_i$, i.e.,
\begin{equation}\label{discussion: event-equal}
    \left\{\mathrm{median}\{\theta_{i,1}^{\mathrm{(on)}}(t),\theta_{i,2}^{\mathrm{(on)}}(t),\hat{\mu}_i^{\mathrm{(on)}}(t)\} > y_i \right\} = \left\{\theta_{i,1}^{\mathrm{(on)}}(t) > y_i\; \cap\;  \theta_{i,2}^{\mathrm{(on)}}(t) > y_i  \right\} \,.\nonumber
\end{equation}
Since $\theta_{i,1}^{\mathrm{(on)}}(t)$, $\theta_{i,2}^{\mathrm{(on)}}(t)$ are i.i.d. conditional on $F_{t-1}$, 
\begin{align}
    \Pr\!\left(\hat{\theta}_i(t)>y_i \mid \mathcal{F}_{t-1}\right)
=\Pr\!\left(\theta^{(\mathrm{on})}_i(t)>y_i \mid \mathcal{F}_{t-1}\right)^2.\label{discussion: sub-twice}
\end{align}

Using \eqref{discussion: sub-twice} squares the Gaussian tail probability (Lemma~\ref{lemma: gaussian ineq}) in the vanilla TS analysis (Lemma~\ref{lemma: vanilla-bad-event-mu}) and yields a halved exploration threshold as
\begin{align}
    \sum_{t=1}^T \Pr\!\left(T_i(t)> \tfrac12 L_i(T),\;G_i(t),\;\hat{\theta}_i(t)>y_i\right)&\le \sum_{t=\tfrac12 L_i(t)}^T \left(\exp\left\{-\frac{T_i(t)(y_i-x_i)^2}{2} \right\}\right)^2 \nonumber\\
    & \le \sum_{t=1}^T  \left(\exp\left\{-\frac{\tfrac{1}{2}L_i(t)(y_i-x_i)^2}{2} \right\}\right)^2 \nonumber \\
&\le \sum_{t=1}^T \frac{1}{T\Delta_i^2}\nonumber \\
&\le \frac{1}{\Delta_i^2}\label{discuss:constant} \,,
\end{align}

where the threshold $L_i(T)$ is the same as in Lemma~\ref{lemma: vanilla-bad-event-mu}.

Thus take \eqref{discuss:constant}, we get
\begin{align}
    \sum_{t=1}^T \Pr(\mathcal{E}_{1,t})
& \le \sum_{t=1}^T \Pr\!\left(A(t)=i,\;T_i(t)\le \tfrac12 L_i(T),\;G_i(t),\;\hat{\theta}_i(t)>y_i\right) \nonumber\\
& \quad \quad \quad+\sum_{t=1}^T \Pr\!\left(A(t)=i,\;T_i(t)>\tfrac12 L_i(T),\;G_i(t),\;\hat{\theta}_i(t)>y_i\right)\nonumber\\
& \le \tfrac12 L_i(T)\;+\;\sum_{t=1}^T \Pr\!\left(T_i(t)>\tfrac12 L_i(T),\;G_i(t),\;\hat{\theta}_i(t)>y_i\right) \nonumber \\
&\le \tfrac12 L_i(T) + \frac{1}{\Delta_i^2} \,, \nonumber
\end{align}
which matches the classical bound but with a factor $1/2$ on the leading logarithmic term.

\paragraph{Optimal-arm term. }
Next we bound the term arising from Lemma~\ref{lemma: median-trans}, namely
\begin{equation}\label{discussion: optimal-ori}
\sum_{t=1}^T\Pr\bigl(\mathcal{E}_{2,t}\bigr) \le \sum_{k=0}^{T-1}
\mathbb E\!\left[
  \frac{1-p_{i,\tau_k+1}}{p_{i,\tau_k+1}}\,\bOne{E_1^{\mu(\mathrm{on})}(\tau_k+1)}
\right]+\sum_{k=0}^{T-1}
\mathbb E\!\left[
  \frac{1-p_{i,\tau_k+1}}{p_{i,\tau_k+1}}\,\bOne{ \neg E_1^{\mu(\mathrm{on})}(\tau_k+1)}
\right]\,. 
\end{equation}

We mainly focus on the first term since the second term upper bound does not depend on $T$. Under $E_1^{\mu(\mathrm{on})}(\tau_k+1)$, we have $\hat{\mu}_1^{\mathrm{(on)}}(\tau_k+1)>y_i$, hence 
$$
p_{i,\tau_k+1}
=\Pr\!\left(\hat{\theta}_1(\tau_k+1)>y_i\mid \mathcal{F}_{\tau_k}\right)
=\Pr\!\left(\theta^{(\mathrm{on})}_{1,1}(\tau_k+1)>y_i \;\cup\; \theta^{(\mathrm{on})}_{1,2}(\tau_k+1)>y_i \mid \mathcal{F}_{\tau_k}\right).
$$

Let $p^{(\mathrm{on})}_{i,\tau_k+1}:=\Pr(\theta^{(\mathrm{on})}_{1}(\tau_k+1)>y_i\mid \mathcal{F}_{\tau_k})$. Then 
\begin{align*}
    p_{i,\tau_k+1}
    & =1-\left(1-p_{i,\tau_k+1}^{\mathrm{(on)}}\right)^2 \\
    & = 2\,p_{i,\tau_k+1}^{\mathrm{(on)}}-\left(p_{i,\tau_k+1}^{\mathrm{(on)}}\right)^2\,. 
\end{align*}

Therefore, 
\begin{equation}\label{discuss: final}
    \frac{1-p_{i,\tau_k+1}}{p_{i,\tau_k+1}}
=\frac{(1-p^{(\mathrm{on})}_{i,\tau_k+1})^2}{p^{(\mathrm{on})}_{i,\tau_k+1}(2-p^{(\mathrm{on})}_{i,\tau_k+1})}
=\frac{1-p^{(\mathrm{on})}_{i,\tau_k+1}}{p^{(\mathrm{on})}_{i,\tau_k+1}}\cdot \frac{1-p^{(\mathrm{on})}_{i,\tau_k+1}}{2-p^{(\mathrm{on})}_{i,\tau_k+1}}
\le \frac12\cdot \frac{1-p^{(\mathrm{on})}_{i,\tau_k+1}}{p^{(\mathrm{on})}_{i,\tau_k+1}}\,, 
\end{equation}
where we used $(1-p)/(2-p)\le 1/2$ for $p\in[0,1]$. 
 Plugging \eqref{discuss: final} into the first term of \eqref{discussion: optimal-ori} directly yields a factor-$1/2$ reduction on the leading logarithmic contribution in the optimal-arm analysis. The second term in \eqref{discussion: optimal-ori} is handled identically to Appendix~\ref{sec:regret} and remains a lower-order constant term.

\subsection{Experiments in the pure online setting}

\begin{figure}[H]
  \centering
  \includegraphics[width=0.7\linewidth]{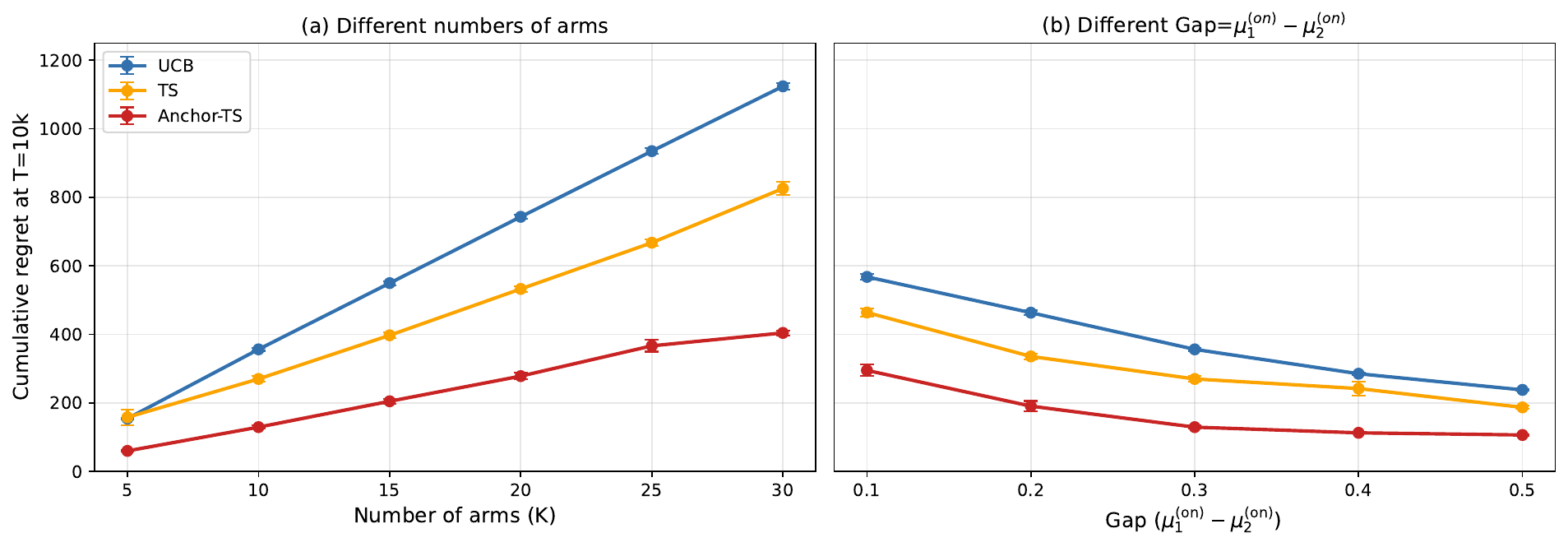}
  \caption{Cumulative regret of our Anchor-TS and baselines in the pure online setting with different arm set size $K$ and sub-optimality gap $\Delta$. }
  \label{fig:fig7}
\end{figure}

We examine the pure online setting in Figure~\ref{fig:fig7}, where no offline data are available. In this case, MIN-UCB, hybrid UCB, and hybrid TS reduce to their respective pure online variants.
The experimental setting follows Section \ref{sec: experiments}. We test different choices of arm set size $K$ and sub-optimality gap $\Delta=\mu_1^{\mathrm{(on)}} - \mu_2^{\mathrm{(on)}}$.
Anchor-TS makes decisions based on the median of two posterior samples and the empirical mean from online data.
As shown in the figure, Anchor-TS consistently achieves the lowest cumulative regret across all tested settings. This demonstrates that even without offline data, the median-based decision rule employed by Anchor-TS effectively suppresses over-exploration caused by the randomness of posterior sampling, thereby exhibiting greater learning efficiency compared with vanilla TS and UCB-based methods.

\section{Proof of Lemmas}
\label{appendix:proof:lemmas}

\begin{proof}[Proof of Lemma \ref{lem:term2-bound}]
    We expand the difference between the empirical hybrid mean and the true mean of the online distribution in $E_i^{\mu(\mathrm{hyb})}(t)$  as follows:
\begin{align}
  \hat\mu_i^{\mathrm{(hyb)}}(t) - \mu^{\mathrm{(on)}}_i
    & = \left(\hat\mu_i^{\mathrm{(hyb)}}(t)-\frac{T_i(t)\mu^{\mathrm{(on)}}_i(t)
           + N_i\mu^{\mathrm{(off)}}_i}
          {T_i(t)+N_i}\right) +\left(\frac{T_i(t)\mu^{\mathrm{(on)}}_i(t)
           + N_i\mu^{\mathrm{(off)}}_i}
          {T_i(t)+N_i}- \mu_i^{\mathrm{(on)}}\right)         \nonumber\\
  & = \left(\hat\mu_i^{\mathrm{(hyb)}}(t)-\frac{T_i(t)\mu^{\mathrm{(on)}}_i(t)
           + N_i\mu^{\mathrm{(off)}}_i}
          {T_i(t)+N_i}\right)+\frac{N_i}{T_i(t)+N_i}
         \bigl(\mu^{\mathrm{(off)}}_i-\mu^{\mathrm{(on)}}_i\bigr).  \nonumber
\end{align}
Using $|\mu^{\mathrm{(off)}}_i-\mu^{\mathrm{(on)}}_i|\le V_i$, we obtain
\begin{equation}
  \bigl|\hat\mu_i^{\mathrm{(hyb)}}(t) - \mu^{\mathrm{(on)}}_i\bigr|
  \le \left|\frac{T_i(t)\hat\mu^{\mathrm{(on)}}_i(t)
           + N_i\hat\mu^{\mathrm{(off)}}_i}
          {T_i(t)+N_i} - \frac{T_i(t)\mu^{\mathrm{(on)}}_i(t)
           + N_i\mu^{\mathrm{(off)}}_i}
          {T_i(t)+N_i}\right| 
     + \frac{N_i V_i}{T_i(t)+N_i}.
  \label{eq:hyb-mean-abs}
\end{equation}
Combining the definition of $\neg E_i^{\mu(\mathrm{hyb})}(t)$ 
\[
  \bigl|\hat\mu_i^{\mathrm{(hyb)}}(t) - \mu^{\mathrm{(on)}}_i\bigr|
  >
  \sqrt{\frac{L_t}{T_i(t)+N_i}} + \frac{N_i V_i}{T_i(t)+N_i},
\]
with \eqref{eq:hyb-mean-abs}, we can derive the following relation:
\begin{equation}
    \neg E_i^{\mu(\mathrm{hyb})}(t) \subseteq \left\{  \left|\frac{T_i(t)\hat\mu^{\mathrm{(on)}}_i(t)
           + N_i\hat\mu^{\mathrm{(off)}}_i}
          {T_i(t)+N_i} - \frac{T_i(t)\mu^{\mathrm{(on)}}_i(t)
           + N_i\mu^{\mathrm{(off)}}_i}
          {T_i(t)+N_i}\right| >  \sqrt{\frac{L_t}{T_i(t)+N_i}}  \right\}.\label{eq:term 2-event} 
\end{equation}
Noting that in (\ref{eq:term 2-event}),
$$
\E\left[\frac{T_i(t)\hat\mu^{\mathrm{(on)}}_i(t)
           + N_i\hat\mu^{\mathrm{(off)}}_i}
          {T_i(t)+N_i}\right] = \frac{T_i(t)\mu^{\mathrm{(on)}}_i(t)
           + N_i\mu^{\mathrm{(off)}}_i}
          {T_i(t)+N_i}.
          $$
Given $T_i(t)=s$,           
by Chernoff-Hoeffding bounds (Lemma~\ref{lemma: chernoff-hoeffding-1}), we have
\begin{align*}
    \Pr\left( \left|\frac{s\hat\mu^{\mathrm{(on)}}_i(t)
           + N_i\hat\mu^{\mathrm{(off)}}_i}
          {s+N_i}-\E\left[\frac{s\hat\mu^{\mathrm{(on)}}_i(t)
           + N_i\hat\mu^{\mathrm{(off)}}_i}
          {s+N_i}\right]\right| >  \sqrt{\frac{L_t}{s+N_i}}     \right) &\leq 2e^{-(s+N_i)\frac{L_t}{s+N_i}} \\ & \le \frac{\delta_t}{t}\,.
\end{align*}
Thus
\begin{align*}
    \Pr\left(\neg E_i^{\mu(\mathrm{hyb})} (t)\right) &= \Pr\left(\exists s \in\{0,\dots,t-1\}:T_i(t)=s\; \text{and}\; \neg E_i^{\mu(\mathrm{hyb})} (t) \right)\\
    &\leq \sum_{s=0}^{t-1}\Pr\left( \left|\frac{s\hat\mu^{\mathrm{(on)}}_i(t)
           + N_i\hat\mu^{\mathrm{(off)}}_i}
          {s+N_i}-\E\left[\frac{s\hat\mu^{\mathrm{(on)}}_i(t)
           + N_i\hat\mu^{\mathrm{(off)}}_i}
          {s+N_i}\right]\right| >  \sqrt{\frac{L_t}{s+N_i}}     \right)\\
      & \le  \sum_{s=0}^{t-1} \frac{\delta_t}{t} \le \delta_t.
\end{align*}
Then we have
\begin{equation}
  \sum_{t=1}^T \Pr(A(t)=i,\;\neg E_i^{\mu(\mathrm{hyb})}(t))\le \sum_{t=1}^{T}\Pr\left(\neg E_i^{\mu(\mathrm{hyb})} (t)\right)\le \sum_{t=1}^{T} \delta_{t}= O\left(\frac{1}{\Delta_i^2}\right)\,. \nonumber
\end{equation}
\end{proof}

\begin{proof}[Proof of Lemma \ref{lemma:term3-bound}]
Recall that
$$
\theta_i^{\mathrm{(hyb)}}(t) \sim \mathcal{N}\!\left(
      \hat\mu^{\mathrm{(hyb)}}_i(t)+\frac{N_i V_i}{T_i(t)+N_i}
     ,
      \frac{1}{T_i(t)+N_i+1}
  \right),
$$
thus $\E[\theta_i^{(\mathrm{hyb})}(t)] = \hat\mu^{\mathrm{(hyb)}}_i(t)+\frac{N_iV_i}{T_i(t)+N_i}$. 
By the triangle inequality, 
$$
\bigl|\theta_i^{(\mathrm{hyb})}(t) - \hat\mu^{\mathrm{(hyb)}}_i(t)\bigr|\le \bigl|\theta_i^{(\mathrm{hyb})}(t) - \E[\theta_i^{(\mathrm{hyb})}(t)]\bigr|+\bigl|\E[\theta_i^{(\mathrm{hyb})}(t)]-\hat\mu^{\mathrm{(hyb)}}_i(t)\bigr|
$$
Thus
\begin{align}
    \neg E_i^{\theta(\mathrm{hyb})}(t)& = \Biggl\{
       \bigl|\theta_i^{(\mathrm{hyb})}(t) - \hat\mu^{\mathrm{(hyb)}}_i(t)\bigr|
       >
       \sqrt{\frac{L_t}{T_i(t)+N_i}}+\frac{N_i V_i}{T_i(t)+N_i}
     \Biggr\}\nonumber
    \\
    &\subseteq \left\{ \bigl|\theta_i^{(\mathrm{hyb})}(t) - \E[\theta_i^{(\mathrm{hyb})}(t)]\bigr|
  \ge \sqrt{\frac{L_t}{T_i(t)+N_i}}\right\}.\label{term(3)}
\end{align}

Conditional on the filtration $\mathcal{F}_t$, 
$\hat\mu^{\mathrm{(hyb)}}_i(t)$ and $T_i(t)$ are deterministic. Thus the right-hand of~\eqref{term(3)} can be bounded by Gaussian tail inequality (Lemma~\ref{lemma: gaussian ineq}) as
$$
\Pr\bigl(\neg E_i^{\theta(\mathrm{hyb})}(t)|\mathcal{F}_t\bigr) \le 2\exp\left(-\frac{1}{2} \frac{L_t}{T_i(t)+N _i}\cdot (T_i(t)+N _i+1)  \right) \le \frac{\delta_t}{t}.
$$

Then we have
\begin{align}
    \sum_{t=1}^T \Pr(A(t)=i,\;\neg E_i^{\mu(\mathrm{hyb})}(t))&= \sum_{t=1}^T \E\left[\Pr(A(t)=i,\;\neg E_i^{\mu(\mathrm{hyb})}(t)|\mathcal{F}_t)\right]\nonumber\\
    &\le   \sum_{t=1}^T \E\left[\Pr(\;\neg E_i^{\mu(\mathrm{hyb})}(t)|\mathcal{F}_t)\right]\nonumber\\
    &\le \sum_{t=1}^T\frac{\delta_t}{t}\nonumber\\
    & = O\left(\frac{1}{\Delta_i^2}\right)\,.\nonumber
\end{align}
\end{proof}

\begin{proof}[Proof of Lemma \ref{lem:goodEvent:sub-optimalArm}]
$\hat\theta_i(t)>y_i$ indicates that the median among three indices is above $y_i$, which
implies that both the online and hybrid indexes are above $y_i$ based on the event $E_i^{\mu(\mathrm{on})}(t)$, i.e., 
\[
  \{\hat\theta_i(t)>y_i\}
  \subseteq
  \{\theta^{\mathrm{(on)}}_i(t)>y_i\}
  \cap
  \{\theta^{\mathrm{(hyb)}}_i(t)>y_i\}\text{ when }\; \mathcal{G}_i(t)\; \text{holds}.
\]
Therefore
\begin{align}
  \Pr(\mathcal{E}_{1,t})
  &= \Pr\bigl(A(t)=i,\; \mathcal{G}_i(t),\; \hat\theta_i(t)>y_i\bigr)
     \nonumber\\
  &\le\min \left\{ \Pr\bigl(A(t)=i,\;\mathcal{G}_i(t),\; \theta^{\mathrm{(on)}}_i(t)>y_i\bigr)
     ,  \Pr\bigl(A(t)=i,\; \mathcal{G}_i(t),\;\theta^{\mathrm{(hyb)}}_i(t)>y_i\bigr)\right\}.
  \nonumber
\end{align}
Then 
\begin{align}
    \sum_{t=1}^T\Pr(\mathcal{E}_{1,t}) &\le \sum_{t=1}^T \min \left\{ \Pr\bigl(A(t)=i,\; \mathcal{G}_i(t),\;\theta^{\mathrm{(on)}}_i(t)>y_i\bigr)
     ,  \Pr\bigl(A(t)=i,\; \mathcal{G}_i(t),\;\theta^{\mathrm{(hyb)}}_i(t)>y_i\bigr)\right\}\nonumber\\
     & \le \min\left\{  \underbrace{\sum_{t=1}^T \Pr\bigl(A(t)=i,\; E_i^{\mu\mathrm{(on)}}(t),\;\theta^{\mathrm{(on)}}_i(t)>y_i\bigr)}_{S_1} ,   \underbrace{\sum_{t=1}^T \Pr\bigl(A(t)=i,\; \mathcal{G}_i(t),\;\theta^{\mathrm{(hyb)}}_i(t)>y_i\bigr)}_{S_2} \right\} \,. \label{eq:4.1 two terms}
\end{align}
Note that $S_{1}$ is exactly the same as the corresponding term in vanilla TS \citep{agrawal2013ts}. By Lemma~\ref{lemma: vanilla-bad-event-theta}, 
\begin{equation}
  S_{1}
  \le
  \frac{2\log(T\Delta_i^2)}{(y_i-x_i)^2}
  + \frac{1}{\Delta_i^2}
  = \frac{18\log(T\Delta_i^2)}{\Delta_i^2}+\frac{1}{\Delta_i^2}
  \, .
  \label{eq:S1:upper}
\end{equation}

Next we bound term $S_{2}$. Based on the good event $\mathcal{G}_{i}(t)$ and triangle inequality, 
\begin{align}
    \theta_i^{(\mathrm{hyb})}(t) - \mu_i^{(\mathrm{on})} & =  \theta_i^{(\mathrm{hyb})}(t)-\mu_i^{(\mathrm{hyb})} + \mu_i^{(\mathrm{hyb})} - \mu_i^{(\mathrm{on})}\nonumber  \\
    & \le 2\sqrt{\frac{L_t}{T_i(t)+N_i}} + \frac{N_iV_i}{T_i(t)+N_i} + \frac{N_i}{T_i(t)+N_i}(\mu^{\mathrm{(off)}}_i-\mu^{\mathrm{(on)}}_i)    \nonumber\\
    & =  2\sqrt{\frac{L_t}{T_i(t)+N_i}}+ \frac{N_i \omega_i}{T_i(t)+N_i}\,.\label{eq:theta-hyb-upper}
\end{align}

Conditional on the event $\{\theta_i^{(\mathrm{hyb})}(t) > y_i\}$, 
\begin{equation}
  \label{eq:theta-hyb-lower}
  \theta_i^{(\mathrm{hyb})}(t) - \mu_i^{(\mathrm{on})}
  > y_i - \mu_i^{(\mathrm{on})}
  = {2\Delta_i}/{3}.
\end{equation}
Combining \eqref{eq:theta-hyb-upper} and \eqref{eq:theta-hyb-lower}, 
\begin{align}
    S_2 \le & \sum_{t=1}^T \mathrm{Pr} \left( A(t)=i,\; 2 \sqrt{\frac{L_t}{T_i(t)+N_i}} + \frac{N_i \omega_i}{T_i(t)+N_i}
  \geq \frac{2\Delta_i}{3} \right) \nonumber  \\
  \le & \sum_{t=1}^T \mathrm{Pr} \left( A(t)=i,\;  2 \sqrt{\frac{L_t}{T_i(t)+N_i}} \geq \frac{\Delta_i}{3}\;\; \text{or} \;\; \frac{N_i \omega_i}{T_i(t)+N_i}\geq \frac{\Delta_i}{3}  \right)  \nonumber \\
  = & \sum_{t=1}^T \mathrm{Pr} \left( A(t)=i,\; T_i(t)+N_i \leq \frac{36 L_t}{\Delta_i^2}\;\; \text{or} \;\; T_i(t)+N_i
  \leq \frac{3 N_i \omega_i}{\Delta_i}    \right)  \nonumber \\ 
  \le & \sum_{t=1}^T \mathrm{Pr} \left( A(t)=i,\; T_i(t)+N_i \leq \max\left\{\frac{36 L_t}{\Delta_i^2},\frac{3 N_i \omega_i}{\Delta_i} \right\}    \right) \nonumber \\
  \le & \left(\max\left\{
    36 \frac{L_T}{\Delta_i^2} ,\,
    3\frac{N_i \omega_i}{\Delta_i}
  \right\}-N_i\right)_+
  \leq
  \left(36 \frac{\log (T\Delta_i^2)}{\Delta_i^2}
  - N_i \Bigl(1 - 3 \frac{\omega_i}{\Delta_i}\Bigr)\right)_+ \label{eq:S_2,upper}
\end{align}

Above all, based on \eqref{eq:4.1 two terms}, \eqref{eq:S1:upper} and \eqref{eq:S_2,upper}, 
\begin{align}
\sum_{t=1}^T \Pr(\mathcal{E}_{1,t}) &  \le \min \left\{S_{1}, S_{2}\right\} \nonumber\\
    & \le\left(
  36\frac{\log (T\Delta_i^2)}{\Delta_i^2}
  - N_i\max\left\{\Bigl(1 - 3\frac{\omega_i}{\Delta_i}\Bigr)
  ,0 \right\}\right)_++ \frac{1}{\Delta_i^2}. 
  \label{eq:term41-bound}   
\end{align}
\end{proof}

\begin{proof}[Proof of Lemma \ref{lemma: median-trans}]
    Our goal is to demonstrate
\begin{equation}
    \Pr\bigl(A(t)=i, \; \hat{\theta}_i(t)\le y_i \mid\mathcal G_i(t),\; F_{t-1}\bigr)
\;\le\;
\frac{1-p_{i,t}}{p_{i,t}}\,
\Pr\bigl(A(t)=1,\; \hat{\theta}_i(t)\le y_i \mid\mathcal G_i(t),\; F_{t-1}\bigr).\label{ineq: sub-to-opt}
\end{equation}

For sub-optimal arm $i$ to be chosen given this constraint, every other arm $j$ must satisfy $\hat{\theta}_j(t)\le \hat{\theta}_i(t)\le y_i$. Besides, noting that given instantiation $F_{t-1}$, the posterior distribution between optimal arm 1 and other sub-optimal arms are independent, and $\Pr\left( \hat{\theta}_1(t)\le y_i \mid \mathcal{G}_i(t),\;F_{t-1}  \right)=\Pr\left( \hat{\theta}_1(t)\le y_i \mid F_{t-1}  \right)$. Therefore,
\begin{align*}
\text{LHS of }\eqref{ineq: sub-to-opt} &\le \Pr\left( \hat{\theta}_j(t)\le y_i,\forall\; j\mid \mathcal G_i(t),\; F_{t-1} \right)\\
    &=\Pr\left( \hat{\theta}_1(t)\le y_i \mid F_{t-1}  \right)\cdot \Pr\left(  \hat{\theta}_j(t)\le y_i,\forall\; j\ne 1\mid \mathcal G_i(t),\; F_{t-1}  \right)\\
    &= (1-p_{i,t})\cdot \Pr\left(  \hat{\theta}_j(t)\le y_i,\forall\; j\ne 1\mid \mathcal G_i(t),\;F_{t-1}  \right).
\end{align*}
    Next, consider the probability of selecting the optimal arm 1 under the same conditions. Arm 1 is chosen if its median index $\hat{\theta}_1(t)$ exceeds all others. Thus we have
    \begin{align*}
\Pr\bigl(A(t)=1\mid  \mathcal G_i(t),\;  F_{t-1}\bigr)&\ge \Pr\left(\hat{\theta}_1(t)>y_i\ge \hat{\theta}_j(t), \forall j\mid \mathcal G_i(t),\; F_{t-1} \right)\\
    &=\Pr\left( \hat{\theta}_1(t)> y_i \mid F_{t-1}  \right)\\
    &\quad\quad\quad \quad \cdot \Pr\left(  \hat{\theta}_j(t)\le y_i, \forall\; j\ne 1 \mid \mathcal G_i(t),\; F_{t-1} \right)\\
    & = p_{i,t} \cdot \Pr\left(  \hat{\theta}_j(t)\le y_i, \forall\; j\ne 1 \mid \mathcal G_i(t),\; F_{t-1}  \right).
\end{align*}
 Combining the above two inequalities, we get the first result in the lemma. 
 
For the term $\mathcal{E}_{2,t}$, taking expectations and summing over $t$ yields
\begin{align}
\sum_{t=1}^T\Pr\bigl(\mathcal{E}_{2,t}\bigr)
&=\sum_{t=1}^T
  \mathbb E \!\Bigl[\Pr\bigl(\mathcal{E}_{2,t}\mid\mathcal F_{t-1}\bigr)\Bigr] \nonumber\\
&\le
\sum_{t=1}^T
\mathbb E\!\left[
  \frac{1-p_{i,t}}{p_{i,t}}\,
  \Pr\bigl(A(t)=1,\mathcal G_i(t)\mid\mathcal F_{t-1}\bigr)
\right] \label{eq:using lemma-optiaml}\\
&= \sum_{t=1}^T
\mathbb E\!\left[
 \E\left[ \frac{1-p_{i,t}}{p_{i,t}}\,
  \bOne{A(t)=1,\mathcal G_i(t)}\mid \mathcal{F}_{t-1}\right]
\right] \label{equality}
\\
&=
\sum_{t=1}^T
\mathbb E\!\left[
  \frac{1-p_{i,t}}{p_{i,t}}\,
  \bOne{A(t)=1,\mathcal G_i(t)}
\right].\label{30}
\end{align}
\eqref{eq:using lemma-optiaml} is derived from the first result in the lemma, \eqref{equality} above uses that $p_{i,t}$ is fixed given $\mathcal{F}_{t-1}$.

Let $\tau_k$ be the time of the $k$-th pull of arm~1 ($k\ge1$) and set
$\tau_0:=0$.  Between two consecutive pulls of arm~1, neither its
posteriors nor their empirical mean estimations change, so both $p_{i,t}$ and
$E_1^{\mu(\mathrm{on})}(t)$ is invariant within each block
$\{\tau_k+1,\dots,\tau_{k+1}\}$.  Summing up the sum in~\eqref{30} block
wise we obtain
\begin{align}
\text{RHS of \eqref{30}}
&\le
\sum_{k=0}^{T-1}
\mathbb E\!\left[
  \frac{1-p_{i,\tau_k+1}}{p_{i,\tau_k+1}}
  \sum_{t=\tau_k+1}^{\tau_{k+1}}
    \bOne{A(t)=1,\mathcal{G}_i(t)}
\right] \nonumber\\
&\le
\sum_{k=0}^{T-1}
\mathbb E\!\left[
  \frac{1-p_{i,\tau_k+1}}{p_{i,\tau_k+1}}\,
\right]\label{eq:optimal ori-term}\\
&\le
\sum_{k=0}^{T-1}
\mathbb E\!\left[
  \frac{1-p_{i,\tau_k+1}}{p_{i,\tau_k+1}}\,\bOne{E_1^{\mu(\mathrm{on})}(\tau_k+1)}
\right]+\sum_{k=0}^{T-1}
\mathbb E\!\left[
  \frac{1-p_{i,\tau_k+1}}{p_{i,\tau_k+1}}\,\bOne{ \neg E_1^{\mu(\mathrm{on})}(\tau_k+1)}
\right], \nonumber 
\end{align}
\eqref{eq:optimal ori-term} is derived because each block contains at most one round with $A(t)=1$ and
$E_1^{\mu(\mathrm{on})}(t)$ coincides with $E_1^{\mu(\mathrm{on})}(\tau_k+1)$ throughout the block.
\end{proof}

\begin{proof}[Proof of Lemma \ref{lemma: optim-good-inter}]
For any time $t$, define $\mathrm{MIX}_t:=\max\{p^{(\mathrm{on})}_{i,t},\,p^{(\mathrm{hyb})}_{i,t}\}$. We first prove that when $E_1^{\mu(\mathrm{on})}(t)$ holds, 
\begin{equation}
p_{i,t}
=\Pr\bigl(\hat\theta_1(t)>y_i\mid\mathcal F_{t-1}\bigr)
\;\ge\;
\mathrm{MIX}_t \,.  \label{32}
\end{equation}

Recall that event $E_1^{\mu(\mathrm{on})}(t) = \{ \hat{\mu}_1^{\mathrm{(on)}}(t) >y_i \}$. This implies with $E_1^{\mu(\mathrm{on})}(t)$ being true,  
\[
\left(\{\theta^{(\mathrm{hyb})}_1(t)>y_i\}\cup\{\theta^{(\mathrm{on})}_1(t)>y_i\}\right)
\subseteq\{\hat\theta_1(t)>y_i\}.
\]
Taking conditional probabilities given $\mathcal F_{t-1}$ and using that
$\theta^{(\mathrm{on})}_1(t)$ and $\theta^{(\mathrm{hyb})}_1(t)$ are
independent of $\hat\mu^{(\mathrm{on})}_1(t)$ conditional on
$\mathcal F_{t-1}$, we obtain
\[
p_{i,t}
\;\ge \Pr\left(\{\theta^{(\mathrm{hyb})}_1(t)>y_i\}\cup\{\theta^{(\mathrm{on})}_1(t)>y_i\}\mid \mathcal F_{t-1}\right)  \ge \;\max\{p^{(\mathrm{on})}_{i,t},p^{(\mathrm{hyb})}_{i,t}\}
=\mathrm{MIX}_t
\text{ on $E_1^{\mu(\mathrm{on})}(t)$}.
\]

Further, using $\frac{1-p}{p}=\frac1p-1$ and the monotonicity of $x\mapsto1/x$
on $(0,1]$, \eqref{32} implies that on $E_1^{\mu(\mathrm{on})}(t)$
\[
\frac{1-p_{i,t}}{p_{i,t}}
\le\frac{1-\mathrm{MIX}_t}{\mathrm{MIX}_t}
=\frac1{\mathrm{MIX}_t}-1,
\]
and multiplying by $\bOne{E_1^{\mu(\mathrm{on})}(t)}$ gets the point-wise bound
\[
\frac{1-p_{i,t}}{p_{i,t}}\bOne{E_1^{\mu(\mathrm{on})}(t)}
\;\le\;\left(\frac1{\mathrm{MIX}_t}-1\right)\bOne{E_1^{\mu(\mathrm{on})}(t)}\le\frac1{\mathrm{MIX}_t}-1 .
\]
The last inequality comes from $\bOne{E_1^{\mu(\mathrm{on})}(t)}\le 1$.

Since $\mathrm{MIX}_t=\max\{p^{(\mathrm{on})}_{i,t},p^{(\mathrm{hyb})}_{i,t}\}$, 
\begin{equation}
\frac1{\mathrm{MIX}_t}-1
\;\le\;
\min\left\{
  \frac1{p^{(\mathrm{on})}_{i,t}}-1,\;
  \frac1{p^{(\mathrm{hyb})}_{i,t}}-1
\right\}.\label{33}
\end{equation}

Applying~\eqref{32}–\eqref{33} at time $\tau_k+1$, we obtain that for each $k$, 
\begin{equation}
\frac{1-p_{i,\tau_k+1}}{p_{i,\tau_k+1}}\bOne{E_1^{\mu(\mathrm{on})}(\tau_k+1)}
\;\le\;
\min\Bigl\{
  \frac1{p^{(\mathrm{on})}_{i,\tau_k+1}}-1,\;
  \frac1{p^{(\mathrm{hyb})}_{i,\tau_k+1}}-1
\Bigr\},\label{34}
\end{equation}
and hence
\begin{equation}
\sum_{k=0}^{T-1}
\mathbb E\!\left[
  \frac{1-p_{i,\tau_k+1}}{p_{i,\tau_k+1}}\,\bOne{E_1^{\mu(\mathrm{on})}(\tau_k+1)}
\right]
\;\le\;
\sum_{k=0}^{T-1}
\min\Bigl\{
  \mathbb E\bigl[\tfrac1{p^{(\mathrm{on})}_{i,\tau_k+1}}-1\bigr],\;
  \mathbb E\bigl[\tfrac1{p^{(\mathrm{hyb})}_{i,\tau_k+1}}-1\bigr]
\Bigr\}.\label{35}
\end{equation}
\end{proof}

\begin{proof}[Proof of Lemma \ref{lemma:hybrid bound-optimal}]
 Recall that $p_{i,t}^{\mathrm{(hyb)}}$ denotes the
probability that $\theta_1^{\mathrm{(hyb)}}(t)$ exceeds $y_i$ given $F_{t-1}$, and for the
algorithm with Gaussian priors we have
$$\theta_1^{\mathrm{(hyb)}}(t)\sim \mathcal{N}\left(\hat\mu^{\mathrm{(hyb)}}_1(t)+\frac{N_1 V_1}{T_1(t)+N_1},\,\frac{1}{T_1(t)+N_1+1}\right).$$

Given $F_{\tau_k}$, let $\Theta_k$ denote a Gaussian random variable  sampled from the above gaussian distribution with $T_1(t)=k$. 
Let $G_k$ be the geometric random variable denoting the number of
consecutive independent trials until a
sample of $\Theta_k$ becomes greater than $y_i$. Then
\[
p_{i,\tau_k+1} \;=\; \Pr(\Theta_k > y_i \mid \mathcal{F}_{\tau_k})
\]
and hence
\[
\mathbb{E}\!\left[\frac{1}{p_{i,\tau_k+1}}\right]
= \mathbb{E}\bigl[\mathbb{E}[G_k \mid \mathcal{F}_{\tau_k}]\bigr]
= \mathbb{E}[G_k].
\]
We first bound $\mathbb{E}[G_k]$ by a constant for all $k$.

Fix any integer $r \ge 1$. Let $z = \sqrt{\ln r}$ and let the random
variable $\operatorname{MAX}_r$ denote the maximum of $r$ independent
samples of $\Theta_k$. For brevity, write
$\hat\mu_1^{\mathrm{(hyb)}} = \hat\mu_1^{\mathrm{(hyb)}}(\tau_k+1)$. Then, for any integer $r \ge 1$,
\begin{align}
\Pr(G_k \le r)
\ge &\Pr(\operatorname{MAX}_r > y_i) \notag\\
\ge &\Pr\!\left(\operatorname{MAX}_r >
\hat\mu_1^{\mathrm{(hyb)}} + \frac{z}{\sqrt{k+N_1+1}}+ \frac{N_1V_1}{k+N_1}\ge y_i\right) \notag\\
= &\mathbb{E}\Bigl[
\mathbb{E}\bigl[\bOne{
\operatorname{MAX}_r >
\hat\mu_1^{\mathrm{(hyb)}} + \frac{z}{\sqrt{k+N_1+1}}+\frac{N_1V_1}{k+N_1} \ge y_i}
\mid \mathcal{F}_{\tau_k}\bigr]\Bigr] \notag\\
= &\mathbb{E}\Bigl[
\bOne{{
\hat\mu_1^{\mathrm{(hyb)}} + \frac{z}{\sqrt{k+N_1+1}}+\frac{N_1V_1}{k+N_1} \ge y_i
}} \cdot \nonumber \\
&\Pr \!\left(\operatorname{MAX}_r >
\hat\mu_1^{\mathrm{(hyb)}} + \frac{z}{\sqrt{k+N_1+1}}+\frac{N_1V_1}{k+N_1}
\mid \mathcal{F}_{\tau_k}\right)
\Bigr]. \label{eq:gj-le-r}
\end{align}

Given $F_{\tau_k}$, $\Theta_k$ is Gaussian with distribution
$\mathcal{N}\left(\hat\mu^{\mathrm{(hyb)}}_1(t)+\frac{N_1 V_1}{k+N_1},\,1/(k+N_1+1)\right)$, using a Gaussian tail (Lemma~\ref{lemma: gaussian ineq}) we have
\begin{align*}
\Pr\!\left(\operatorname{MAX}_r >
\hat\mu_1^{\mathrm{(hyb)}} + \frac{z}{\sqrt{k+N_1+1}}+\frac{N_1V_1}{k+N_1}
\;\middle|\; \mathcal{F}_{\tau_k}\right)
&\ge 1 -
\left(1 - \frac{1}{\sqrt{2\pi}}\,
\frac{z}{z^2+1}\,
e^{-z^2/2}\right)^{\!r} \\
&= 1 - \left(
1 - \frac{1}{\sqrt{2\pi}}\,
\frac{\sqrt{\ln r}}{\ln r + 1}\,
\frac{1}{\sqrt{r}}
\right)^{\!r}.
\end{align*}
For sufficiently large $r$ (in particular, for all
$r \ge e^{11}$), we have
\[
\left(
1 - \frac{1}{\sqrt{2\pi}}\,
\frac{\sqrt{\ln r}}{\ln r + 1}\,
\frac{1}{\sqrt{r}}
\right)^{\!r}
\le \exp\!\left(-\frac{\sqrt{r}}{4\pi r \ln r}\right)
\le \frac{1}{r^2},
\]
and hence for all $r \ge e^{11}$,
\begin{equation}\label{eq:maxr-lower}
\Pr\!\left(\operatorname{MAX}_r >
\hat\mu_1^{\mathrm{(hyb)}} + \frac{z}{\sqrt{j+N_1+1}}+\frac{N_1V_1}{j+N_1}
\;\middle|\; \mathcal{F}_{\tau_j}\right)
\;\ge\; 1 - \frac{1}{r^2}.
\end{equation}

Substituting \eqref{eq:maxr-lower} into \eqref{eq:gj-le-r} yields, for
$r \ge e^{11}$,
\begin{equation}\label{eq:gj-le-r-2}
\Pr(G_k \le r)
\;\ge\;
\left(1 - \frac{1}{r^2}\right)
\Pr\!\left(
\hat\mu_1^{\mathrm{(hyb)}} + \frac{z}{\sqrt{k+N_1+1}}+\frac{N_1V_1}{k+N_1} \ge y_i
\right).\nonumber
\end{equation}

Next we obtain a lower bound on the second term using Chernoff--Hoeffding
bound (Lemma~\ref{lemma: chernoff-hoeffding-1}). 
\begin{align}
    \Pr\!\left(
\hat\mu_1^{\mathrm{(hyb)}} + \frac{z}{\sqrt{k+N_1+1}}+\frac{N_1V_1}{k+N_1} \ge y_i
\right)
 \ge &\Pr\!\left(
\hat\mu_1^{\mathrm{(hyb)}} + \frac{z}{\sqrt{k+N_1+1}}+\frac{N_1V_1}{k+N_1} \ge \mu_1^{\mathrm{(on)}}
\right)  \nonumber \\
\ge &  \Pr\!\left(\hat\mu_1^{\mathrm{(hyb)}}+
 \frac{z}{\sqrt{k+N_1+1}} \ge \frac{k\cdot \mu^{\mathrm{(on)}}_1
           + N_1\mu^{\mathrm{(off)}}_1}
          {k+N_1}\right) \nonumber\,,
\end{align}
where the last inequality is from a simple decomposition for the hybrid mean like (\ref{eq:hyb-mean-abs}). 

Applying the Chernoff bound at time $t = \tau_k + 1$ (so that
$k_1(t) = k$), we get,
for any $x>0$,
\[
\Pr\!\left(
\hat\mu_1^{\mathrm{(hyb)}} + \frac{1}{k+N_1} + \frac{x}{\sqrt{k+N_1+1}} \ge \mu_1^{\mathrm{(hyb)}}
\right)
\ge 1 - e^{-2x^2}.
\]
Here, the term $\frac{1}{k+N_1+1}$ was added to $\hat\mu_1^{\mathrm{(hyb)}}$ to adjust for the fact that $\hat\mu_1^{\mathrm{(hyb)}}$ is not simply average of the past $k+N_1$ samples, instead, it is the sum of past $k+N_1$ samples divided by $k+N_1+1$. Now, we use $x:= z-\frac{1}{\sqrt{k+N_1+1}}\ge z-\frac{1}{\sqrt{N_1+1}}$for all $k\ge 0$, we obtain
\begin{align*}
    \Pr\!\left(
\hat\mu_1^{\mathrm{(hyb)}} + \frac{z}{\sqrt{k+N_1+1}}+\frac{N_1V_1}{k+N_1} \ge \mu_1
\right)
&\ge 1 - \exp{\left(-2\left(z-\frac{1}{\sqrt{N_1+1}}\right)^2\right)}\\
& = 1 - \frac{1}{r^2}\exp{\left(\frac{4}{\sqrt{N_1+1}}\sqrt{\log r}-\frac{2}{N_1+1}\right)}.
\end{align*}

Since $y_i \le \mu_1$, this further implies
\begin{equation}\label{eq:mu1hat-ge-yi}
\Pr\!\left(
\hat\mu_1^{\mathrm{(hyb)}} + \frac{z}{\sqrt{k+N_1+1}} +\frac{N_1V_1}{k+N_1}\ge y_i
\right)
\ge 1 - \frac{1}{r^2}\exp{\left(\frac{4}{\sqrt{N_1+1}}\sqrt{\log r}-\frac{2}{N_1+1}\right)}.
\end{equation}

Noting that for $r\ge \exp\left(\frac{28+16\sqrt{3}}{N_1+1} \right)$, we have 
\begin{equation}
    \frac{1}{r^2}\exp{\left(\frac{4}{\sqrt{N_1+1}}\sqrt{\log r}-\frac{2}{N_1+1}\right)}\le \frac{1}{r^{1.5}} \label{eq: key constant 1}
\end{equation}

and combining this with \eqref{eq:gj-le-r-2},
\eqref{eq:mu1hat-ge-yi} and \eqref{eq: key constant 1} yields, for all $r \ge \max\left\{e^{11}, \exp\left(\frac{28+16\sqrt{3}}{N_1+1}\right)\right\}$,
\[
\Pr(G_k \le r) \ge 1 - \frac{1}{r^2} - \frac{1}{r^{1.5}}.
\]

We can now bound $\mathbb{E}[G_k]$:
\begin{align}
\mathbb{E}[G_k]
&= \sum_{r=0}^{\infty} \Pr(G_k \ge r) \notag\\
&= 1 + \sum_{r=1}^{\infty} \Pr(G_k \ge r) \notag\\
&\le 1 + e^{11}
+ \sum_{r\ge 1}
\left(\frac{1}{r^2} + \frac{1}{r^{1.5}}\right) \notag\\
&\le 1 + \max\left\{e^{11}, \exp\left(\frac{28+16\sqrt{3}}{N_1+1}\right)\right\} + 2 + 2.7\nonumber
\\\;&\le\; \max\left\{e^{11}, \exp\left(\frac{28+16\sqrt{3}}{N_1+1}\right)\right\} + 6. \label{eq:Gj-const}
\end{align}
Hence
\[
\mathbb{E}\!\left[\frac{1}{p_{i,\tau_k+1}} - 1\right]
= \mathbb{E}[G_k] - 1
\le \max\left\{e^{11}, \exp\left(\frac{28+16\sqrt{3}}{N_1+1}\right)\right\} + 5
\quad\text{for all } k.
\]

Next we derive a tighter bound for large $k$.   $$k > L_i^{\mathrm{(hyb)}}(T) = \left(\frac{288\ln(T\Delta_i^2 + e^{6})}{\Delta_i^2} -N_1\right)_+.$$
Again fix $r \ge 1$, let $z=\sqrt{\ln r}$, and define $\operatorname{MAX}_r$
as before. Then
\begin{align}
\Pr(G_k \le r)
&\ge \Pr(\operatorname{MAX}_r > y_i) \notag\\
&\ge \Pr\!\left(
\operatorname{MAX}_r >
\hat\mu_1^{\mathrm{(hyb)}} + \frac{z}{\sqrt{k+N_1+1}}+\frac{N_1V_1}{k+N_1} - \frac{\Delta_i}{6} \ge y_i
\right) \notag\\
&= \mathbb{E}\bigl[
\bOne{
\hat\mu_1^{\mathrm{(hyb)}} + \frac{z}{\sqrt{k+N_1+1}}+\frac{N_1V_1}{k+N_1} + \frac{\Delta_i}{6}
\ge \mu_1^{\mathrm{(on)}}
}\\
&\quad\quad\quad \qquad \cdot\Pr\!\left(
\operatorname{MAX}_r >
\hat\mu_1^{\mathrm{(hyb)}} + \frac{z}{\sqrt{k+N_1+1}}+\frac{N_1V_1}{k+N_1} - \frac{\Delta_i}{6}
\;\middle|\; \mathcal{F}_{\tau_k}
\right)
\bigr], \label{eq:gj-large-j}
\end{align}
where we used that $y_i = \mu_1 - \Delta_i/3$.

By the definition of $L_i^{\mathrm{(hyb)}}(T)$ we have
\[
k+N_1 \;\ge\; \frac{288\ln(T\Delta_i^2+e^6)}{\Delta_i^2}
,
\]
and hence
\[
\frac{\sqrt{2\ln(T\Delta^2+e^6)}}{\sqrt{k+N_1+1}}
\le \frac{\Delta_i}{12}.
\]
Therefore, for all $r \le (T\Delta_i^2 + e^6)^2$,
\[
\frac{z}{\sqrt{k+N_1+1}}-\frac{\Delta_i}{6}
= \frac{\sqrt{\log r}}{\sqrt{k+N_1+1}} - \frac{\Delta_i}{6}
\le -\frac{\Delta_i}{12}.
\]
Using the upper tail bound for a Gaussian random variable
(Lemma~\ref{lemma: gaussian ineq}) we obtain, for any realization $F_{\tau_k}$,
\[
\Pr\!\left(
\Theta_k >
\hat\mu_1^{\mathrm{(hyb)}}(\tau_k+1)+\frac{N_1V_1}{k+N_1} - \frac{\Delta_i}{12}
\;\middle|\; F_{\tau_k}
\right)
\ge 1 - \frac{1}{2}
\exp\!\left(-\frac{(k+1)\Delta_i^2}{288}\right)
\ge 1 - \frac{1}{2(T\Delta_i^2 + e^{32})},
\]
which implies
\[
\Pr\!\left(
\operatorname{MAX}_r >
\hat\mu_1^{\mathrm{(hyb)}}(\tau_k+1) + \frac{z}{\sqrt{k+N_1+1}} +\frac{N_1V_1}{k+N_1} - \frac{\Delta_i}{6}
\;\middle|\; F_{\tau_k}
\right)
\ge 1 -  \frac{1}{2^r(T\Delta_i^2 + e^{6})^r}.
\]

Moreover, for any $t \ge \tau_k+1$ we have $T_1(t) \ge k$, and by
Chernoff--Hoeffding bound (Lemma~\ref{lemma: chernoff-hoeffding-1}),

\begin{align*}
    \Pr\!\left(
\hat\mu_1^{\mathrm{(hyb)}}(t) + \frac{z}{\sqrt{k+N_1+1}}+\frac{N_1V_1}{k+N_1} - \frac{\Delta_i}{6}
\ge y_i
\right)
&\ge \Pr\!\left(
\hat\mu_1^{\mathrm{(hyb)}}(t) \ge \mu_1^{\mathrm{(hyb)}} - \frac{\Delta_i}{6}
\right)\\
 &   \ge 1 - \exp\!\left(-\frac{2(T_1(t)+N_1)\Delta_i^2}{36}\right)\\
 &\ge 1 - \frac{1}{(T\Delta_i^2 + e^{6})^{16}}.
\end{align*}

Let $T' = (T\Delta_i^2+e^6)^2$. Therefore,
\begin{align*}
    \E[G_k]& \le \sum_{r=1}^\infty \Pr(G_k\ge r)\\
    & \le 1+\sum_{r=1}^{T'} \Pr(G_k\ge r) +\sum_{r=T'+1}^{\infty} \Pr(G_k\ge r)\\
    &\le 1+\sum_{r=1}^{T'} \left(\frac{1}{(2\sqrt{T'})^r}+\frac{1}{(T')^8}\right)+\sum_{r=T'+1}^{\infty} \left(\frac{1}{r^2}+\frac{1}{r^{1.5}}\right)\\ &
    \le 1+\frac{1}{\sqrt{T'}}+\frac{1}{(T')^7}+\frac{2}{T'} +\frac{3}{\sqrt{T'}}\\
    &\le 1+\frac{5}{T\Delta_i^2+e^6}.
\end{align*}
The above upper bound shows that $\E\left[\frac{1}{p_{i,\tau_k+1}^{\mathrm{(hyb)}}} \right]-1 =\E\left[G_k\right]-1\le \frac{5}{T\Delta_i^2}$ for $k > L_i^{\mathrm{(hyb)}}(T)$.
\end{proof}

\begin{proof}[Proof of Lemma \ref{lemma: optim-bad-term}]
On $E_1^{\mu(\mathrm{on})}(\tau_k+1)$ we have $\hat\mu^{(\mathrm{on})}_1(\tau_k+1)\le y_i$, hence

\begin{equation}
    \Big\{\operatorname{median}\{\theta^{(\mathrm{on})}_1(\tau_k{+}1),\,\hat\mu^{(\mathrm{on})}_1(\tau_k{+}1),\,\theta^{(\mathrm{hyb})}_1(\tau_k{+}1)\}>y_i\Big\}
=
\Big\{\theta^{(\mathrm{on})}_1(\tau_k{+}1)>y_i,\ \theta^{(\mathrm{hyb})}_1(\tau_k{+}1)>y_i\Big\}.\label{eq: median-trans-bad}
\end{equation}

We denote the right-hand of \eqref{eq: median-trans-bad} as
\[
p_k^{\wedge}
:=\Pr\!\left(\theta^{(\mathrm{on})}_1(\tau_k{+}1)>y_i,\ \theta^{(\mathrm{hyb})}_1(\tau_k{+}1)>y_i \mid \mathcal F_{\tau_k}\right).
\]

Conditioned on $\mathcal F_{\tau_k}$, draw i.i.d. copies
\[
\big(\Theta^{(\mathrm{on})}_{k,\ell},\Theta^{(\mathrm{hyb})}_{k,\ell}\big)_{\ell\ge 1}
\stackrel{\text{i.i.d.}}{\sim}
\big(\theta^{(\mathrm{on})}_1(\tau_k{+}1),\,\theta^{(\mathrm{hyb})}_1(\tau_k{+}1)\big)\mid \mathcal F_{\tau_k},
\]
and define the (conditional) geometric hitting time
\[
G_k := \min\Big\{\ell\ge 1:\ \Theta^{(\mathrm{on})}_{k,\ell}>y_i,\ \Theta^{(\mathrm{hyb})}_{k,\ell}>y_i\Big\}.
\]
Then $G_k\mid\mathcal F_k\sim\mathrm{Geom}(p_k^{\wedge})$ (supported on $\{1,2,\dots\}$), so
\[
\frac{1}{p_k^{\wedge}}=\mathbb E[G_k\mid \mathcal F_k].
\]

Fix any $\alpha\in(1,\tfrac32)$ and let $\beta := \frac{\alpha}{\alpha-1}$. By Jensen's inequality (since $x\mapsto x^\alpha$ is convex for $\alpha>1$),
\[
\Big(\frac{1}{p_k^{\wedge}}\Big)^\alpha
=\big(\mathbb E[G_k\mid \mathcal F_k]\big)^\alpha
\le \mathbb E[G_k^\alpha\mid \mathcal F_k],
\qquad\Rightarrow\qquad
\mathbb E\Big[\Big(\frac{1}{p_k^{\wedge}}\Big)^\alpha\Big]
\le \mathbb E[G_k^\alpha].
\]

Next we bound $\mathbb E[G_k^\alpha]$ uniformly in $k$.
For any integer $r\ge 1$,
\[
\{G_k>r\}
\subseteq
\Big\{\max_{1\le \ell\le r}\Theta^{(\mathrm{on})}_{k,\ell}\le y_i\Big\}
\ \cup\
\Big\{\max_{1\le \ell\le r}\Theta^{(\mathrm{hyb})}_{k,\ell}\le y_i\Big\}.
\]

Recall in Lemma~\ref{lemma: vanilla optim-term} and Lemma~\ref{lemma:hybrid bound-optimal} where we prove a constant upper bound for all $k\ge0$, for $r\ge r_0= e^{64}$, we have
\begin{align*}
    \Pr\left(\max_{1\le \ell\le r}\Theta^{(\mathrm{on})}_{k,\ell}\le y_i\right)&\le \frac{1}{r^2}+\frac{1}{r^{3/2}}\, ,\\
    \Pr\left(\max_{1\le \ell\le r}\Theta^{(\mathrm{hyb})}_{k,\ell}\le y_i\right)&\le \frac{1}{r^2}+\frac{1}{r^{3/2}}\,.   
\end{align*}

Therefore,

\[
\Pr(G_k>r) \le \frac{2}{r^2}+\frac{2}{r^{3/2}} \le \frac{4}{r^{3/2}}.
\]
For integer-valued $X\ge 1$, we first note that
$$
(r+1)^\alpha -r^\alpha \le \alpha (r+1)^{\alpha-1}\le \alpha 2^{\alpha-1}r^{\alpha-1},
$$
Then we have
\begin{align*}
    X^\alpha &= \sum_{r=1}^{X-1} \left((r+1)^\alpha -r^\alpha\right)+1\\
    & \le \sum_{r=1}^{X} \left((r+1)^\alpha -r^\alpha\right)\\
    &\le \alpha 2^{\alpha-1}  \sum_{r=1}^{X}  r^{\alpha-1}\\
    & \le \alpha 2^{\alpha-1}  \sum_{r=1}^{\infty}  r^{\alpha-1} \bOne{ r\le X }\\
    &=\alpha 2^{\alpha-1}  \sum_{r=1}^{\infty}  r^{\alpha-1} \bOne{  X\ge r }
\end{align*}

Taking exception on both sides, let $C_\alpha = \alpha2^{\alpha-1}$,we have

\[
\mathbb E[X^\alpha]\le C_\alpha \sum_{r=1}^\infty r^{\alpha-1}\mathbb P(X\ge r),
\]
hence with $X=G_k$,
\[
\mathbb E[G_k^\alpha]
\le C_\alpha\left( \sum_{r=1}^{r_0-1}r^{\alpha-1} +\sum_{r=r_0}^\infty r^{\alpha-1}\cdot \frac{4}{r^{3/2}}\right)
= \,C_\alpha\left(C_{r_0}+4 \sum_{r=r_0}^\infty \,r^{\alpha-\frac{5}{2}}
\right)\le M_\alpha <\infty,
\]
where the last inequality uses $\alpha<\tfrac{3}{2}$ so that $\alpha-\tfrac{5}{2}<-1$, and we use $C_{r_0}$ to denote $\sum_{r=1}^{r_0-1} r^{\alpha-1}$ as a constant. 

Consequently,
\begin{equation}\label{eq:uniform-moment}
\sup_{k\ge 0}\ \mathbb E\Big[\Big(\frac{1}{p_k^{\wedge}}\Big)^\alpha\Big]
\le M_\alpha <\infty.
\end{equation}

By H\"older's inequality (Lemma~\ref{lemma: HOLDER ine}) with exponents $(\alpha,\beta)$,
\begin{align*}
    \E\Big[\frac{1}{p_{i,\tau_k+1}} \bOne{\neg E_1^{\mu(\mathrm{on})}(\tau_k+1)}\Big]&
=\E\Big[\frac{1}{p_k^{\wedge}}    
\bOne{\neg E_1^{\mu(\mathrm{on})}(\tau_k+1)}\Big]\\
&\le
\Big(\E\Big[\Big(\frac{1}{p_k^{\wedge}}\Big)^\alpha\Big]\Big)^{1/\alpha}
\cdot \Pr(\neg E_1^{\mu(\mathrm{on})}(\tau_k+1))^{1/\beta}
\\& \le
M_\alpha^{1/\alpha}\,  \Pr(\neg E_1^{\mu(\mathrm{on})}(\tau_k+1))^{1/\beta}.
\end{align*}
Noting that, by approximation, taking $r_0=e^{64}$, $(C_{r_0})^{1/\alpha}\approx e^{64}$.

Finally, since at time $\tau_k+1$ the online empirical mean of arm~$1$ is computed from exactly $k$
i.i.d. Gaussian samples with mean $\mu_1$ and variance $1$, we have the standard tail bound
\[
\Pr(\neg E_1^{\mu(\mathrm{on})}(\tau_k+1))
=\mathbb P\big(\hat\mu^{(\mathrm{on})}_{1,k}\le y_i\big)
\le \exp\Big(-\frac{k(\mu_1-y_i)^2}{2}\Big)
= \exp\Big(-\frac{k\Delta_i^2}{18}\Big),
\]
where we used $y_i=\mu_1-\Delta_i/3$.
Therefore,
\begin{equation}
    \sum_{k=1}^{T-1}\E\Big[\frac{1}{p_{i,\tau_k+1}} \bOne{\neg E_1^{\mu(\mathrm{on})}(\tau_k+1)}\Big]
\le
M_\alpha^{1/\alpha}\sum_{k=1}^{\infty}\exp\Big(-\frac{k\Delta_i^2}{18\beta}\Big)
\le \frac{C}{\Delta_i^2},\label{eq:optimal const-term}
\end{equation}
for some constant $C>0$ depending only on $\alpha$ (hence on $\beta$), which completes the bound
for the $\1\{\neg E_1^{\mu(\mathrm{on})}(\tau_k+1)\}$-part of Lemma \ref{lemma: median-trans}. 
\end{proof}

\section{Useful Lemmas}

\begin{lemma}[Chernoff-Hoeffding Bound]\label{lemma: chernoff-hoeffding-1} Let $X_1,\dots,X_n$ be independent random variables in $[0,1]$ with $\E[X_i]=\mu_i$ (not necessarily equal). Let $X=\frac{1}{n}\sum_{i=1}^n X_i$, $\mu=\E[X] = \frac{1}{n}\sum_{i=1}^n \mu_i$. Then for any $0<\epsilon<1-\mu$,
$$
\Pr(X\ge \mu+\epsilon)\le e^{-2n\epsilon^2},
$$
and, for any $0<\epsilon<\mu$,
$$
\Pr(X\le \mu-\epsilon )\le e^{-2n\epsilon^2}.
$$
\end{lemma}

\begin{lemma}[Concentration inequality of Gaussian variables \citep{abramowitz2006handbook}]\label{lemma: gaussian ineq} For a Gaussian distributed random variable $Z$ with mean $\mu$ and variance $\sigma^2$, for any $x>0$,

$$
\Pr(Z>\mu+x\sigma)\ge \frac{1}{\sqrt{2\pi}}\frac{x}{x^2+1}e^{-x^2/2}.
$$
\end{lemma}

\begin{lemma}[H\"older's inequality]\label{lemma: HOLDER ine}
    Let $p,q\in(1,\infty)$ be conjugate exponents such that $\frac{1}{p}+\frac{1}{q}=1$. Let $X$ and $Y$ be real-valued random variables defined on the same probability space. If $\E[|X|^p]< \infty$ and $\E[|Y|^q]< \infty$, then

    $$
    \E[|XY|]\le (\E[|X|^p])^{\frac{1}{p}}(\E[|Y|^q])^{\frac{1}{q}}.
    $$
    
\end{lemma}

\begin{lemma}[Lemma 2.15 in \cite{agrawal2013ts}]\label{lemma: vanilla-bad-event-mu} For any sub-optimal arm $i\ne 1$, 
$$
\sum_{t=1}^T \Pr(A(t)=i, \hat{\mu}_i^{\mathrm{(on)}}(t)\le x_i)
 \le \frac{1}{d(x_i,y_i)}+1\le \frac{9}{2\Delta_i^2}+1,
$$    
where $d(a, b)=a \ln \frac{a}{b}+(1-a) \ln \frac{(1-a)}{(1-b)}$.

\end{lemma}

\begin{lemma}[Lemma~2.16 in \cite{agrawal2013ts}]\label{lemma: vanilla-bad-event-theta} For any sub-optimal arm $i\ne 1$,
$$
\sum_{t=1}^T \Pr(A(t)=i, \theta_i^{\mathrm{(on)}}(t) >y_i, \hat{\mu}_i^{\mathrm{(on)}}(t) \le x_i)
 \le L_i(T)+\frac{1}{\Delta_i^2},
$$
where $L_i(T)\ge \frac{2\log(T\Delta_i^2)}{(y_i-x_i)^2}$.
\end{lemma}

\begin{lemma}[Lemma~2.13 in \cite{agrawal2013ts}]\label{lemma: vanilla optim-term} Let $\tau_j$ denote the time of the $j^{th}$ pull of the optimal arm 1. Then
$$
\mathbb{E}\left[\frac{1}{p_{i,\tau_j+1}^{\mathrm{(on)}}}-1\right] \leq\left\{\begin{array}{ll}
e^{64}+5 & \forall j, \\
\frac{5}{T \Delta_{i}^{2}}, & j>L_{i}(T),
\end{array}\right.
$$
where $L^{\mathrm{(on)}}_i(T)=\frac{288\log(T\Delta_i^2+e^{32})}{\Delta_i^2}$.
\end{lemma}

\end{document}